\documentclass[sigconf]{acmart}

\usepackage{times}
\usepackage{helvet}
\usepackage{courier}
\usepackage{amsmath}
\usepackage{bm}
\usepackage{booktabs}
\usepackage{float}
\usepackage{amsmath}
\usepackage{graphicx}  
\usepackage{bm}
\usepackage{algorithm, algorithmic}
\usepackage{caption}
\usepackage{subfigure}
\usepackage{makecell}
\usepackage{wrapfig}
\usepackage{stfloats}
\usepackage{bbm}
\usepackage{multirow}
\usepackage{makecell}
\usepackage{colortbl}
\usepackage{enumitem}
\usepackage{empheq}

\newtheorem{myDef}{Definition}
\newtheorem{myLem}{Lemma}

\newtheorem{mytheory}{Theorem}
\newtheorem{myPro}{Proposition}

\AtBeginDocument{%
	\providecommand\BibTeX{{%
			\normalfont B\kern-0.5em{\scshape i\kern-0.25em b}\kern-0.8em\TeX}}}

\copyrightyear{2023} 
\acmYear{2023} 
\setcopyright{acmlicensed}\acmConference[WWW '23]{Proceedings of the ACM Web Conference 2023}{May 1--5, 2023}{Austin, TX, USA}
\acmBooktitle{Proceedings of the ACM Web Conference 2023 (WWW '23), May 1--5, 2023, Austin, TX, USA}
\acmPrice{15.00}
\acmDOI{10.1145/3543507.3583335}
\acmISBN{978-1-4503-9416-1/23/04}
\acmSubmissionID{3779}

\begin{document}
	
	\title{Robust Mid-Pass Filtering Graph Convolutional Networks}
	
	\author{Jincheng Huang}
	\authornote{Work performed during the internship at MSRA}
	\affiliation{%
		\institution{School of Computer Science, Southwest Petroleum University}
		\city{Chengdu}
		\country{China}
	}
    \email{huangjc0429@gmail.com}

	\author{Lun Du}
	\authornote{Corresponding Author}
	\affiliation{%
		\institution{Microsoft Research Asia}
		\city{Beijing}
		\country{China}
	}
    \email{lun.du@microsoft.com}
    
	\author{Xu Chen}
	\affiliation{%
		\institution{Microsoft Research Asia}
		\city{Beijing}
		\country{China}
	}
    \email{xu.chen@microsoft.com}
	
	\author{Qiang Fu}
	\affiliation{%
		\institution{Microsoft Research Asia}
		\city{Beijing}
		\country{China}
	}
    \email{qifu@microsoft.com}
	
	\author{Shi Han}
	\affiliation{%
		\institution{Microsoft Research Asia}
		\city{Beijing}
		\country{China}
	}
    \email{shihan@microsoft.com}
	
	\author{Dongmei Zhang}
	\affiliation{%
		\institution{Microsoft Research Asia}
		\city{Beijing}
		\country{China}
	}
    \email{dongmeiz@microsoft.com}
	
	
	
	\begin{abstract}
Graph convolutional networks (GCNs) are currently the most promising paradigm for dealing with graph-structure data, while recent studies have also shown that GCNs are vulnerable to adversarial attacks. Thus developing GCN models that are robust to such attacks become a hot research topic.

However, the structural purification learning-based or robustness constraints-based defense GCN methods are usually designed for specific data or attacks, and introduce additional objective that is not for classification. Extra training overhead is also required in their design.
To address these challenges, we conduct in-depth explorations on mid-frequency signals on graphs and propose a simple yet effective \textbf{Mid}-pass filter \textbf{GCN} (\textbf{Mid-GCN}). Theoretical analyses guarantee the robustness of signals through the mid-pass filter, and we also shed light on the properties of different frequency signals under adversarial attacks. Extensive experiments on six benchmark graph data further verify the effectiveness of our designed Mid-GCN in node classification accuracy compared to state-of-the-art GCNs under various adversarial attack strategies.

\end{abstract}

\ccsdesc[500]{Mathematics of computing~Spectra of graphs}
\ccsdesc[500]{Computing methodologies~Artificial intelligence}
	
	\keywords{graph neural networks, spectral graph neural networks, robustness, adversarial attacks, node classification}
	

\maketitle
	
\section{Introduction}
Recently, graph convolutional neural networks (GCNs) have achieved promising performance in graphs, which bridge the gap between traditional structured-data-oriented deep learning methods and graph-structured data. Therefore GCNs are applied in a broad range of fields, such as molecular prediction~\cite{molecular2,molecular}, traffic prediction~\cite{traffic,traffic2}, web search~\cite{graph_web,graph_web2} and so on. However, a host of recent studies have shown that GCNs are vulnerable to adversarial attacks~\cite{nettack,metattack,attack_1,attack_2}, and the lack of robustness may lead to security and privacy problems. For instance, under criminal accomplice prediction on social networks, in order to disguise themselves, accomplices can directly trade with people with high credit, which leads to pollution of first-order neighbors of the high credit group and reduce their credibility. Similarly, with the explosive growth of network resources, a tremendous amount of illegal advertisements and phishing websites are flooding the world wide web. These malicious web pages can pursue a high search engine recommendations ranking through links from high-quality websites (with some black hat SEO technology.) Such cases can be of great challenges to GCNs; therefore, how to develop GCNs capable of resisting adversarial attacks is a particularly critical issue.

Current adversarial attack methods on GCNs can be divided into targeted attacks and non-targeted attacks. Targeted attacks mainly focus on fooling GCNs to misclassify target nodes, and the most representative method, \emph{nettack}~\cite{nettack}, generates unnoticeable perturbations by preserving degree distribution and imposing constraints on feature co-occurrence. As for non-targeted attacks, attackers intend to reduce the overall test set accuracy, and \emph{mettack}~\cite{metattack} is the classic one that generates poisoning attacks based on meta-learning.
Revisiting these effective defense GCNs we discover that existing studies pay more attention to restoring a clean graph, that is, restoring deleted original edges and deleting adversarial edges~\cite{gcn-jaccard,gcn-svd,prognn,zhang2020gnnguard,nagraph,huang2022learning}. These methods require additional plug-ins and introduce some inductive biases, which cause a series of problems: (1) The elaborate network architecture or optimization objective of mainstream defense models is based on specific data or attack paradigms, leading to unstable and even degraded performance on different datasets or attacks. For example, GCN-SVD~\cite{gcn-svd} is designed for \emph{nettack} attacks, and will performance worse after the \emph{mettack} attacks. 
(2) The robust regularization term or model component of defense models aims to defend adversarial attacks~\cite{rgcn,gcn-jaccard,gcn-svd,prognn,LFR}, which is inconsistent with the desirable node classification prediction accuracy, and thus the learned representations are not the optimal ones for this task. (3) Model complexity increases with the additional network components or regularization objective, and it also raises the burden on computation resources as well as the consuming time. Since there are many drawbacks of extra design for defending the adversarial attacks, a potential solution would be re-design GCNs, more precisely, elaborate a graph filter that is capable of preserving more attack-insensitive information.

To fill this gap, we conduct an experiment to investigate the impact of adversarial attacks on the eigenvalues of real-word citation graphs as graph filter is always associated with its eigenvalues. 
From Figure~\ref{fig:spectral} we can observe that \textbf{eigenvalues of two real-graph within mid-frequency are less affected by adversarial attacks.}
In this paper, the theoretical analysis provides us an insight that mid-frequency signals tend to preserve information from higher-order neighbors.
Considering attacks on graphs, direct perturbations on one-hop neighbors of nodes are always more effective than on multi-hop neighbors ~\cite{zhu2022does}. Besides, the exponential growth of the number of high-order neighbors also increases the difficulty of attacks. It indicates the potential benefit of defending against adversarial attacks on GCNs by preserving information from neighbors of higher order, i.e., by utilizing mid-frequency signals.

\begin{figure}[!ht]
	\centering
	\subfigure[Cora]{
		\begin{minipage}[t]{0.5\linewidth} 
			\centering
			\includegraphics[scale=0.27]{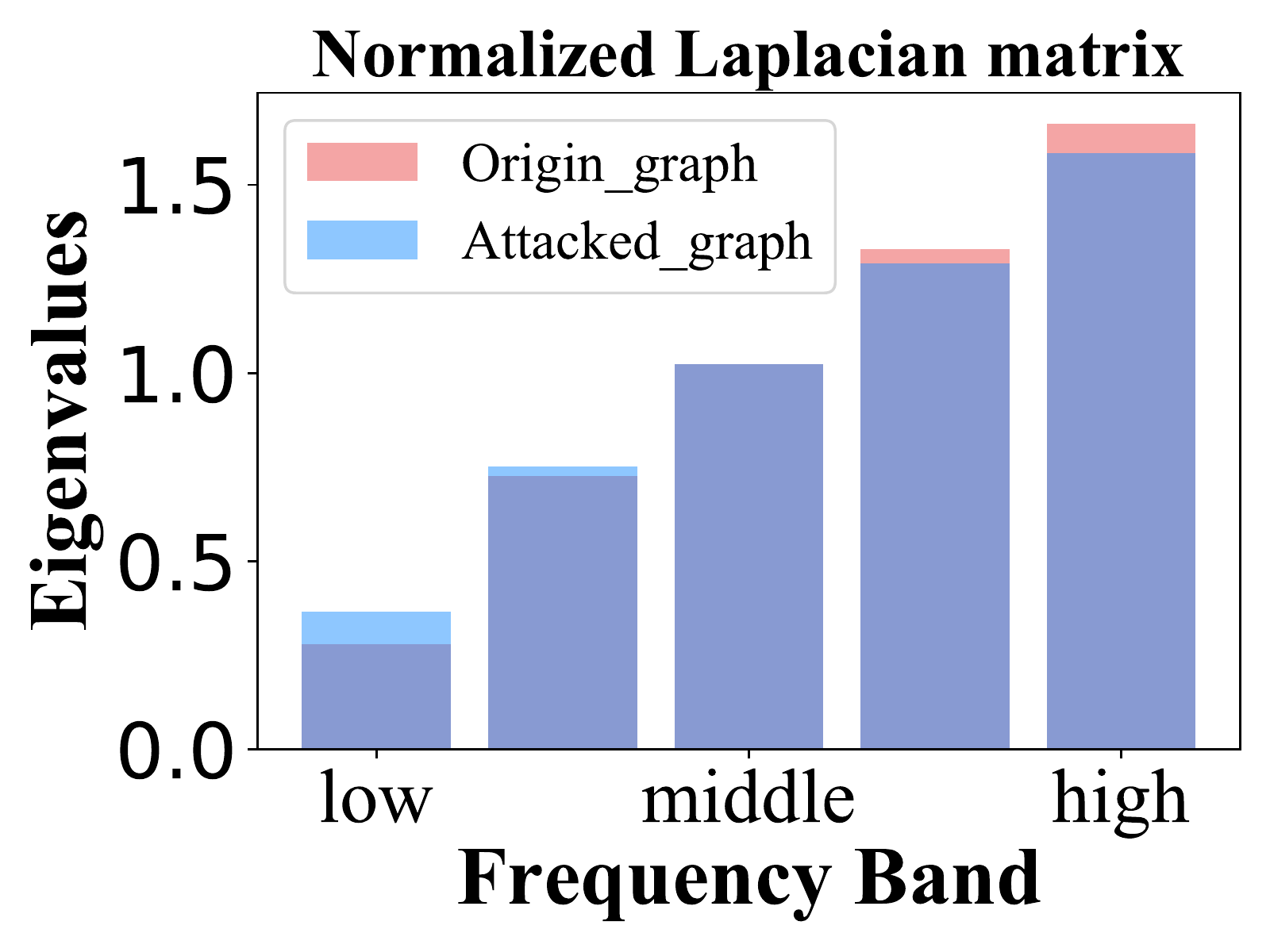} 
		\end{minipage}%
	}\subfigure[Citeseer]{
		\begin{minipage}[t]{0.5\linewidth}
			\centering
			\includegraphics[scale=0.27]{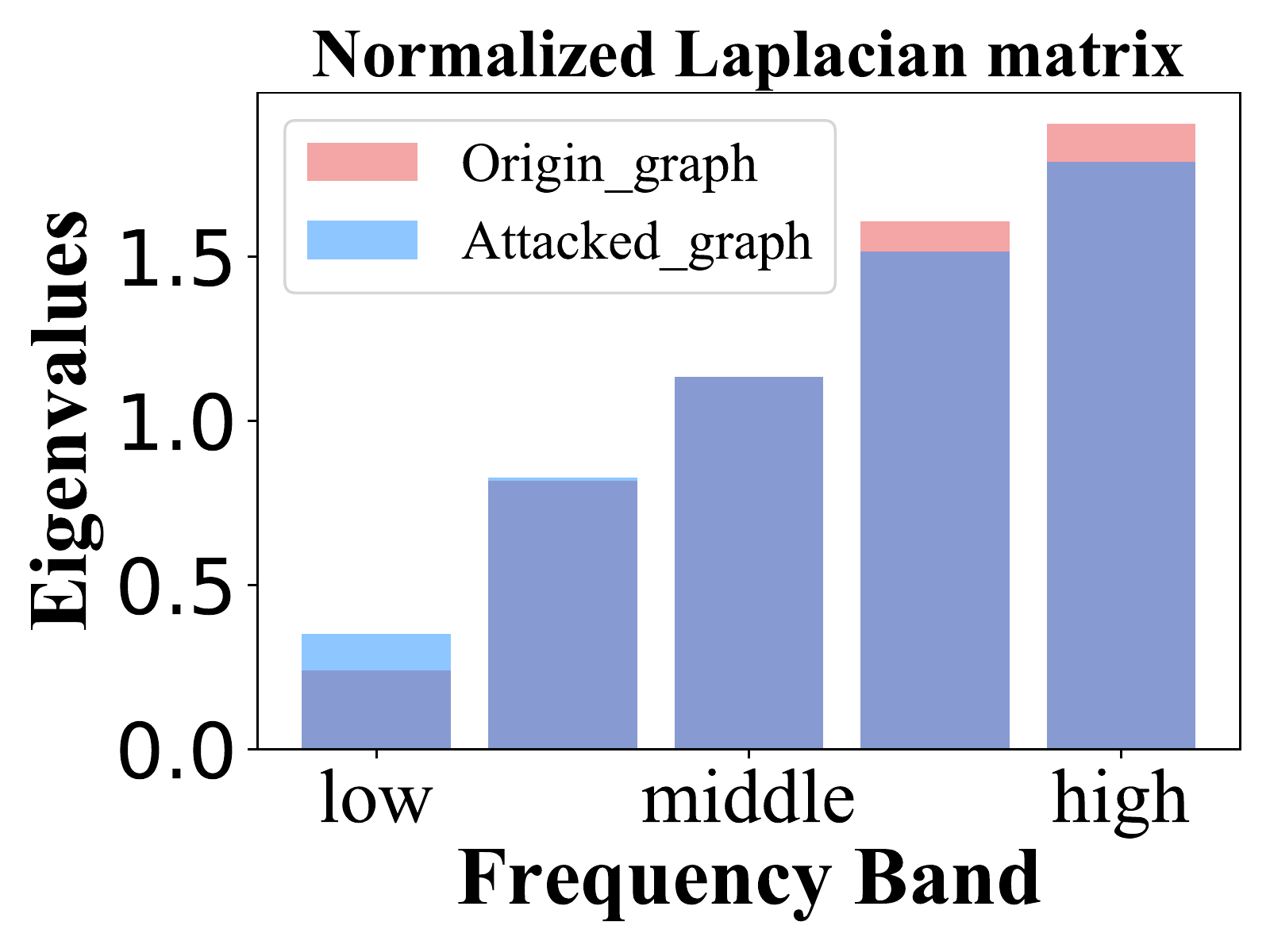}
		\end{minipage}%
	}
	\caption{Eigenvalues variation before and after \emph{metattack} on Cora and Citeseer.}
	\label{fig:spectral}
\end{figure}

Hence, we study the mid-frequency signal of GCN and propose a general mid-pass filtering GCN (\textbf{Mid-GCN}) in this paper. 
Unlike higher-order polynomial filtering-based spectral GCNs with many parameters and hyperparameters that may cause over-fitting, or first-order information-based GCN such as vanilla GCN only low-frequency signals~\cite{gcn_lowpass} are exploited and FAGCN~\cite{fagcn} weighting high-frequency and low-frequency signals,
Mid-GCN focuses on the mid-frequency signals whose Laplacian eigenvalue is around 1. We theoretically verify the robustness of Mid-GCN under structural perturbations. Extensive experiments on various graph datasets under different adversarial attacks also exhibit the superiority of our proposed Mid-GCN over other state-of-the-art defense methods. More surprisingly, Mid-GCN outperforms vanilla GCN for the node classification task on unperturbed graphs.

The contributions of this paper are summarized as follows:
\begin{itemize}
    \item This paper, to the best of our knowledge, is the first study that explores mid-frequency signals of GCNs from the perspective of robustness, and a critical inspiration is that \emph{mid-frequency signals are not susceptible to adversarial attacks}.
    \item we design a mid-pass filtering GCN model \emph{Mid-GCN} due to the robustness of middle-frequency signals against adversarial attacks.
    We theoretically prove such a property of the mid-frequency signal-based GCN.
    \item Extensive experimental results demonstrate the effectiveness of our proposed Mid-GCN under various adversarial attack strategies.

\end{itemize}

\section{Related Work}
In this section, we introduce related work from spectral graph neural networks and defense models.\\

\subsection{Spectral Graph Neural Network.} Spectral graph neural network is a main type of graph neural network, which mainly focuses on the filter design of the eigenvalues of the Laplacian matrix~\cite{spectral_lecun}. Earlier, \cite{cheyGCN2016} used Chebyshev polynomials to fit the filter shape, enabling fast convolution of spectral neural networks. On this basis, GCN~\cite{kipf2017gcn} truncates the first two Chebyshev polynomials to further simplify the entire operation and bring improved performance and many subsequent work~\cite{2019sgc,appnp,heat-gcn,ma2021unified} on spectral graph neural networks are based on GCN. Scattering GCN~\cite{scatterGCN} uses wavelet matrix to approximate the bandpass filter to assist solve the problem of over-smoothing on the graph. GPRGNN~\cite{GPRGNN2021} directly learns the parameters of the polynomial, i.e., learns the shape of the filter. BernNet~\cite{bernnet} and JacobiConv~\cite{jacobi-net} use different polynomials with fortunate properties to learn approximate filters. The above methods can obtain various types of filters through polynomials of different orders, including low-pass filter, high-pass filter, band-pass filter, band-stop filter, etc. However, they contain too many parameters and hyperparameters, which will lead to slow hyperparameter optimization, overfitting and other risks. Recently EvenNet~\cite{evennet} found that even-hop neighbors are more robust, so they proposed an even-polynomial graph filter. However, the current methods do not pay attention to the 
middle frequency signal. This article is the first to discover the power of the middle frequency signal in the defense adversarial attack.\\

\subsection{Adversarial Attacks and Defense for GNNs.}In the field of deep learning, model robustness and adversarial attack/defense have always been an important part of exploring model capabilities. Adversarial attacks are deliberately designed perturbations to minimize the performance of the model and these perturbations are often not easily detected. Due to the special graph structure of GNN, the previous methods cannot be used in adversarial attacks, so many works have explored the adversarial attack of graph structures~\cite{nettack,metattack,deeprobust,attack_1,attack_2,attack_2018,zheng2021graph} . Attack methods can be divided into two types: 1. Targeted attack~\cite{nettack,attack_2018}. For a specific node attack, misleading the model to classify the target node, the most representative method is \emph{nettack}~\cite{nettack}, which generates imperceptible perturbations by preserving degree distribution and feature co-occurrence. 2.Non-targeted attack~\cite{metattack}. It is not attacking a certain node, but the goal is to make the overall effect on the test set worse, \emph{metattack}~\cite{metattack} attacks graph structure through meta-learning. For the above-mentioned adversarial attacks, researchers also pay attention to how the GNN model can cope with these attacks. We divide defense methods into two categories, 
\begin{itemize}[leftmargin=*]
    \item \textbf{Graph structure purification methods.} This kind of method is based on some properties of the graph to restore the original graph structure, such as GCN-Jaccard~\cite{gcn-jaccard} purifies graphs by Jaccard similarity based on feature homophily assumption. CLD-GNN~\cite{nagraph} diffuses the training set labels into the test set through label propagation, then preserves the edges between nodes of the same class and deletes edges between different node classes. GCN-SVD~\cite{gcn-svd} found that the nettack attack only affects the part of the adjacency matrix with smaller singular values. Naturally, the noise can be removed by approximating the original adjacency matrix with low rank.
    and state-of-the-art model Pro-GNN~\cite{prognn} constrains the graph structure to be low-rank and sparse through regular terms, so that the graph structure is close to the real structure. The same as Pro-GNN is GNNGUARD~\cite{zhang2020gnnguard}, which  detects the adverse effects existing in the relationship between the graph
    structure and node features by neighbor importance estimation and layer-wise graph memory.

    \item \textbf{Robust network methods.} This type of method usually designs a robust network structure that can well defend against the impact of adversarial attacks. A typical method is RGCN~\cite{rgcn}, which represents the node features with a Gaussian distribution to absorb the effects of adversarial changes. LFR~\cite{LFR} is a spectral method to bring the robustness information from eligible low-frequency components in the spectral domain.
\end{itemize}


        \section{Preliminaries}
\label{sec:prelimit}
\textbf{Notations.} Let $\mathcal{G} = (V,E)$ be an undirected and unweighted graph with node set $V$ and edge set $E$. The nodes are described by a feature matrix~ $ X \in\mathbb{R}^{n\times f}$, where $n$ is the number of nodes and $f$ is the number of features for each node. One node is associated with a class label that is depicted in the label matrix $Y \in\mathbb{R}^{n\times c}$ with a total of $c$ classes. We represent the graph by its adjacency matrix $A \in \mathbb{R}^{n\times n}$ and the graph with self-loops can be denoted as $\tilde{A} = A + I_{n}$.\\

\noindent\textbf{Graph Fourier Transform and Graph Signals.} In the light of graph signal theory~\cite{graph_fourier}, graph convolution is equivalent to the Laplacian transform of graph signals from the spatial domain to the spectral domain. Let $L = I_{n} - D^{-\frac{1}{2}}AD^{-\frac{1}{2}}$ be the normalized graph Laplacian matrix, which is positive semi-definite and has a complete set of orthogonal eigenvectors $\{u_{l}\}^{n}_{l=1}$ with $u_{l}\in \mathbb{R}^{n}$ corresponding to the eigenvalue $\lambda_{l} \in [0, 2]$. Specifically, $L=U\Lambda U^T$ where $\Lambda=\operatorname{diag}([\lambda_{1},\ldots,\lambda_{n}])$.
Eigenvectors of the Laplacian matrix are analogous to the basis functions of Fourier transform that can transform signals from the spatial domain to the spectral domain 
as $\hat{x} = U^{T}x$, and the inverse transform is $x = U\hat{x}$. Hence, the convolution operation over graph signal $x$ with the kernel $f$ can be formulated as:
\begin{equation}
	\label{eq:graph_four}
 f \ast_{\mathcal{G}} x =
g_{\theta}\star x = U g_{\theta} U^T x,
\end{equation}
where $g_{\theta}=\operatorname{diag}(\theta)$ is a diagonal matrix that make modifications on the spectral domain.

\section{Methodology}
In this section, we describe our designed mid-pass filtering GCN and theoretically analyze the benefits of mid-frequency signals in defending against adversarial attacks in detail. Some cases are also provided to intuitively demonstrate the advantages of mid-frequency signals.
\label{sec:method}
\subsection{Mid-GCN: Mid-Pass Filtering GCN}
\label{sec:model}
As depicted in Figure~\ref{fig:filter}, low-pass filter suppresses the high-frequency signals and enhances the low eigenvalues while the high-pass filter inhibits the low-frequency signals and the high eigenvalues gains. Intuitively, to achieve the maximum value when $\lambda$ is around 1, the filter should be a quadratic function-like shape. Hence, we design a simple mid-pass filter as:
\begin{equation}
    \begin{aligned}
g_{\theta} &= \Lambda(2I-\Lambda), \\
\mathcal{F} &= (I-D^{-\frac{1}{2}}AD^{-\frac{1}{2}})(I + D^{-\frac{1}{2}}AD^{-\frac{1}{2}}).
\end{aligned}
\end{equation}

\begin{figure}[htpb]
	\centering
	\includegraphics[width=0.5\textwidth]{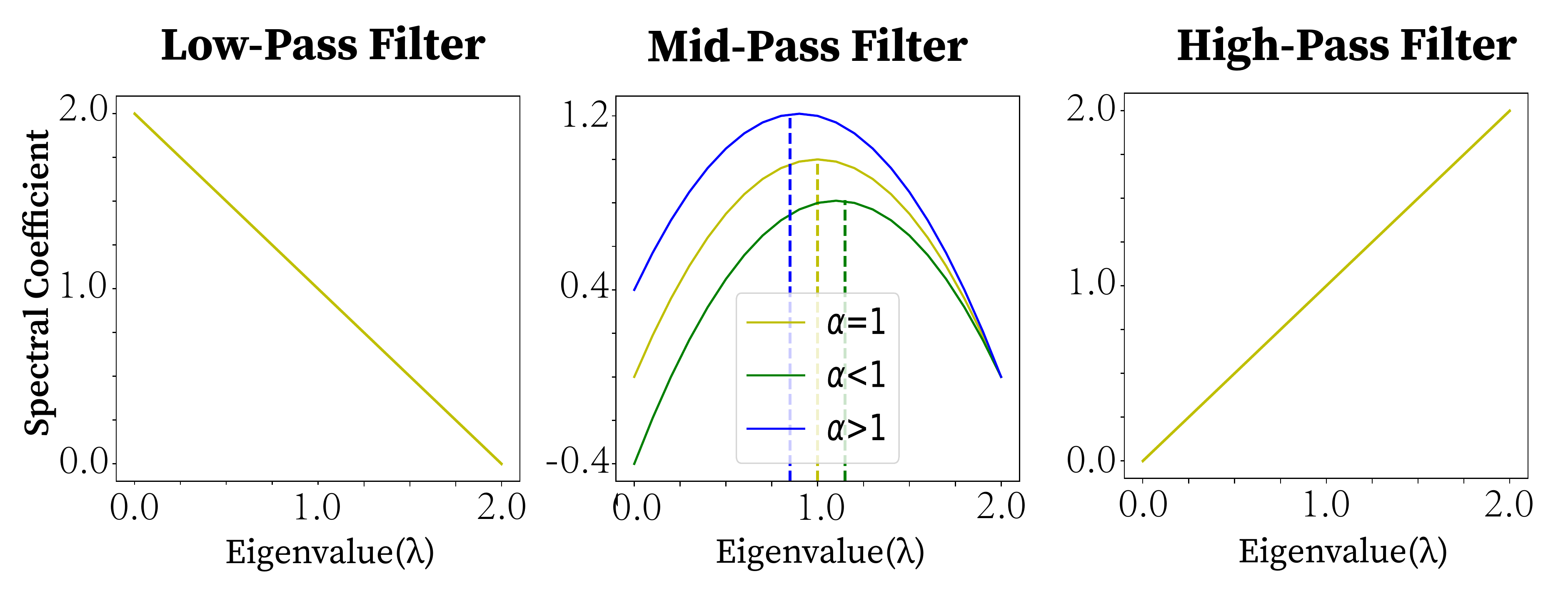}
	\caption{Shape of low-pass, mid-pass and high-pass filter.}
	\label{fig:filter}
\end{figure}

We further increase the model flexibility by introducing a 
hyper-parameter $\alpha$ to balance between the high-frequency and low-frequency signals, and the Mid-pass filtering GCN can be denoted as:

\begin{equation}
\label{eq:mid-gnn}
    Z = (\alpha I-D^{-\frac{1}{2}}AD^{-\frac{1}{2}})(I + D^{-\frac{1}{2}}AD^{-\frac{1}{2}})XW,
\end{equation}
with $\alpha \in [0, 2]$. When $\alpha \rightarrow 0$, it tends to be a high-pass filter, and it will become a low-pass filter if $\alpha \rightarrow 2$. $W \in \mathbb{R}^{f \times h}$ is learnable parameters for feature transformation over signals. Figure~\ref{fig:filter} also provides a more intuitive image of the mid-pass filter. 

\vspace{-0.5cm}
\begin{figure}[htpb]
	\centering
	\subfigure[Cora]{
		\begin{minipage}[t]{0.5\linewidth} 
			\centering
			\includegraphics[scale=0.1]{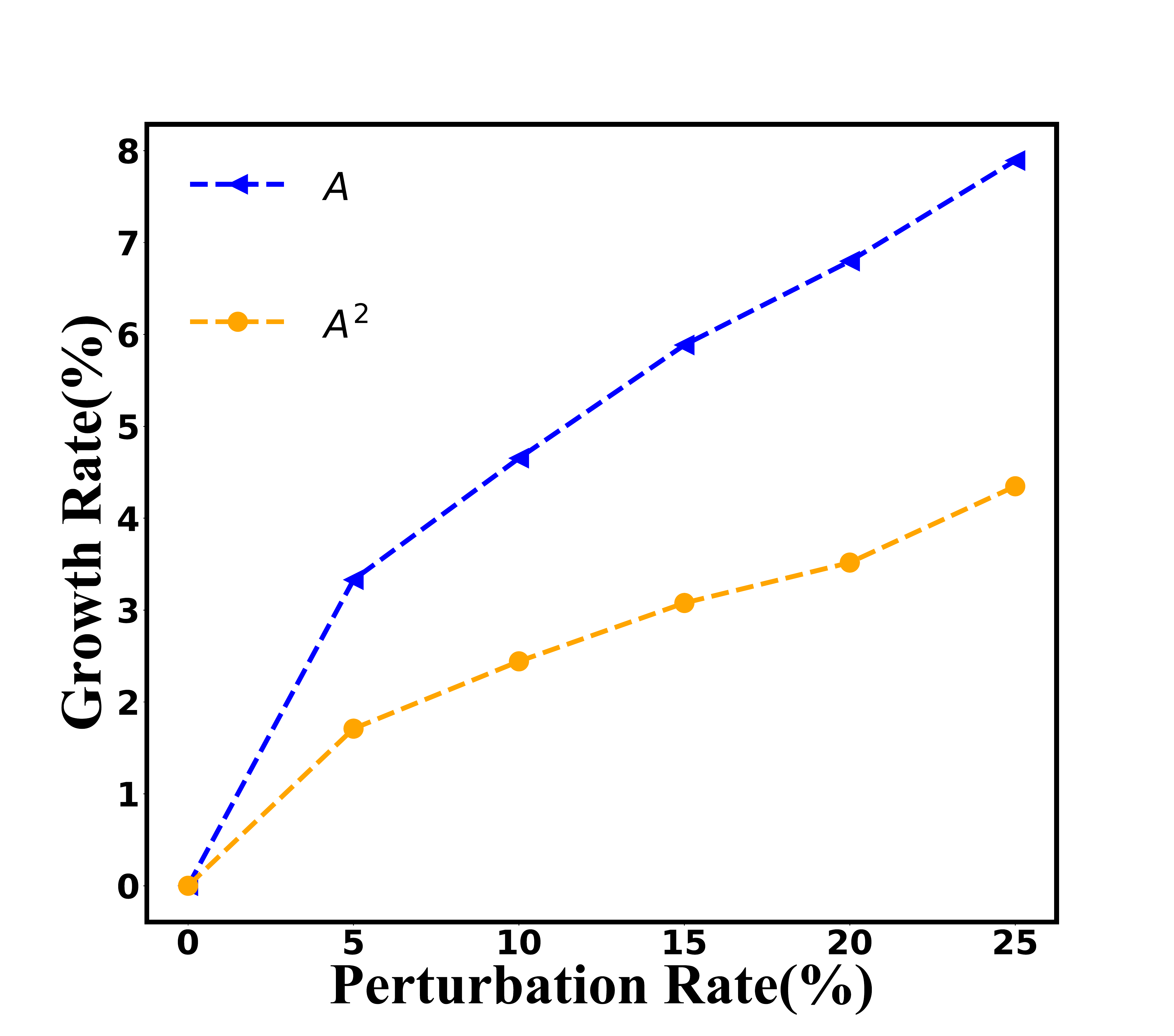} 
		\end{minipage}%
	}\subfigure[Citeseer]{
		\begin{minipage}[t]{0.5\linewidth}
			\centering
			\includegraphics[scale=0.1]{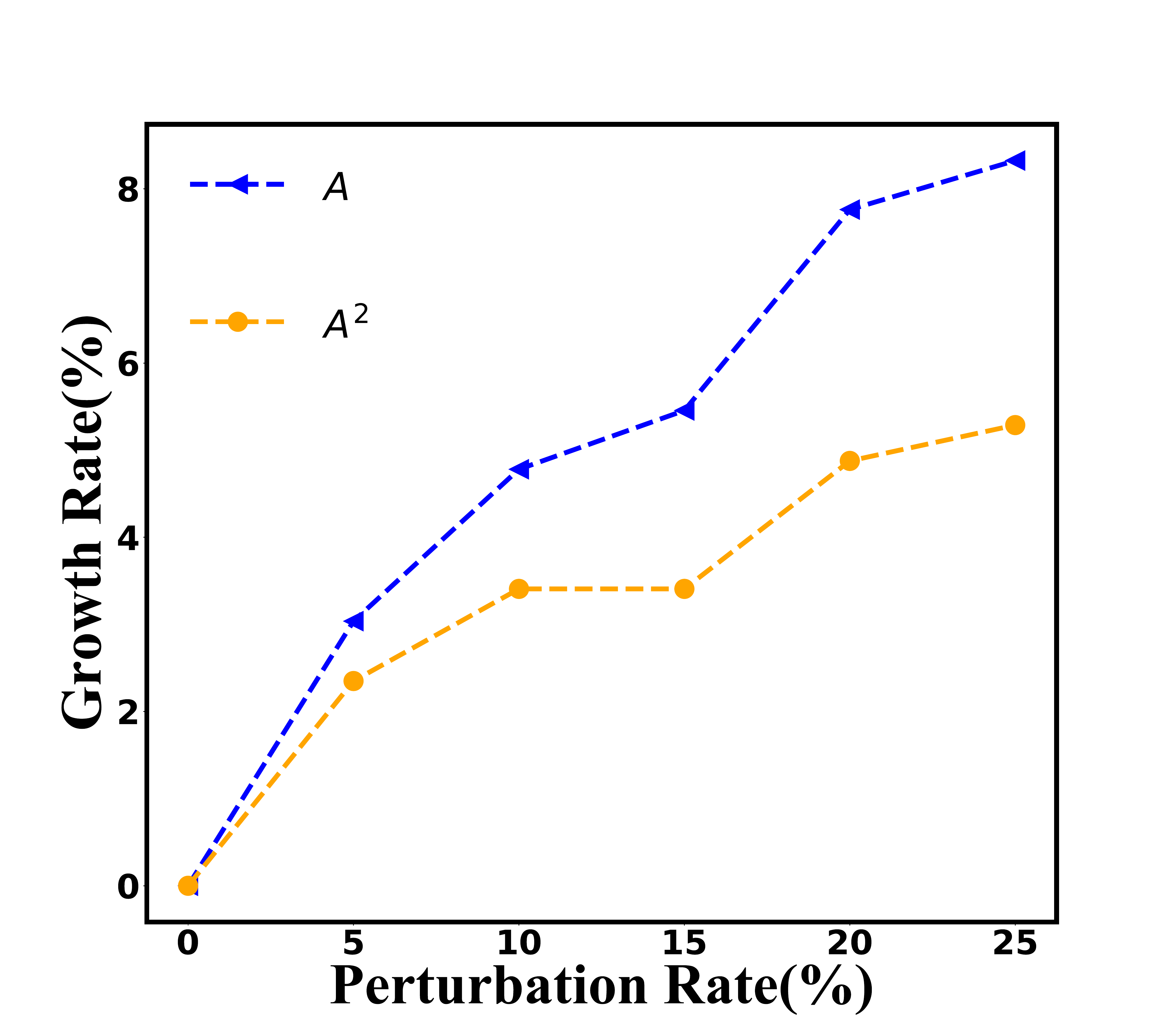}
		\end{minipage}%
	}
	\caption{Rank growth of $A$ and $A^2$ under \emph{metattack}.}
	\label{fig:rank}
\end{figure}

\noindent\textbf{Rank growth of $A$ and $A^2$ under attacks.} Previous work~\cite{prognn} shows that real-world graph structures are generally low-rank, and rank of its adjacency matrix $A$ increases as the \emph{metattack} perturbation rate rises. Since the graph filter of GCN is the first order of $A$ and that of Mid-GCN is $A^2$, we apply attacks to Cora and Citeseer to see the variations of the rank of $A$ and $A^2$ in Figure~\ref{fig:rank}. It inspires us that rank of higher-order graph filter is more robust to metattack. Moreover, higher-order filters capture information from higher-order neighbors and can avoid the potential risk of perturbations on first-order neighbors, which is consistent with our original purpose and also shares similar ideas with some existing research~\cite{zhu2022does}. 

\subsection{Analysis of Structural Attack Influence on Node Representation Distance}
\label{sec:theo}

As mentioned in existing studies~\cite{nettack,zhu2022does}, direct perturbations on one-hop neighbors of nodes are more effective than perturbations on multi-hop neighbors. Therefore, we conduct an in-depth analysis to explore the influence of signals within different frequency bands on the distance of first-order neighbors. Apart from the low-pass filter $\mathcal{F}_{L}$ from GCN and the mid-pass filter $\mathcal{F}_{M}$ of our Mid-GCN, We design a high-pass filter $\mathcal{F}_{H}$ as:
\begin{equation}
	\begin{aligned}
	&\mathcal{F}_{L} =I +  D^{-\frac{1}{2}}AD^{-\frac{1}{2}}=2I-L \\
	&\mathcal{F}_{H} = I-D^{-\frac{1}{2}}AD^{-\frac{1}{2}}=L \\
	&\mathcal{F}_{M} = (I-D^{-\frac{1}{2}}AD^{-\frac{1}{2}})(I + D^{-\frac{1}{2}}AD^{-\frac{1}{2}})=L(2I-L), \\
	\end{aligned}
\end{equation}

According to the above filters, the processed graph signals can be written as follows:

\begin{equation}
	\begin{aligned}
	&\mathcal{F}_{L} \ast_\mathcal{G} x = U[2I-\Lambda]U^{T}x  \\
	&\mathcal{F}_{H}\ast_\mathcal{G}x = U[\Lambda]U^{T}x\\
	&\mathcal{F}_{M}\ast_\mathcal{G}x = U[\Lambda(2I-\Lambda)]U^{T}x ,
	\end{aligned}
\end{equation}

Previous work~\cite{fagcn} have analyzed the impact of high-pass filter and low-pass filter from the perspective of the distance between representation of connected nodes, and we further investigate the effect of the mid-pass filter on this measure. Suppose there is a pair of connected nodes $(u, v)$ on the graph, and their representations are $h_{u}$ and $ h_{v}$, respectively. Euclidean distance of two nodes can be calculated as $\mathcal{D} = \left \|h_{u} - h_{v} \right \|_{2}$. Therefore, we denote $\mathcal{D}_{L}$, $\mathcal{D}_{H}$, and $\mathcal{D}_{M}$ as the node distances of low-frequency/high-frequency/mid-frequency signals.

\begin{equation}
\label{eq:dis}
\begin{aligned}
\mathcal{D}_{L} &= \left \|(h_{u} + \frac{1}{\sqrt{d_{u}}\sqrt{d_{v}}}h_{v}) - (h_{v} + \frac{1}{\sqrt{d_{u}}\sqrt{d_{v}}}h_{u}) \right \|_{2} \\ &= (1-\frac{1}{\sqrt{d_{u}}\sqrt{d_{v}}})\mathcal{D} \\
\mathcal{D}_{H} &= \left \|(h_{u} - \frac{1}{\sqrt{d_{u}}\sqrt{d_{v}}}h_{v}) - (h_{v} - \frac{1}{\sqrt{d_{u}}\sqrt{d_{v}}}h_{u}) \right \|_{2} \\ &= (1+\frac{1}{\sqrt{d_{u}}\sqrt{d_{v}}})\mathcal{D} \\
\mathcal{D}_{M} &= \begin{cases}
     (1+\frac{1}{\sqrt{d_{u}}\sqrt{d_{v}}} \sum_{t}\frac{1}{d_{t}})\mathcal{D}  &\quad\text{if }  t\in \mathcal{N}_{u} \,\text{and } t\in \mathcal{N}_{v}, \\ 
    \mathcal{D}   & \quad \text{otherwise.}
\end{cases}
\end{aligned}
\end{equation}
where $t$ is the node connecting $u$, $v$ and $d_{t}$ is the degree of $t$, $d_{t} \geq 2$. $\mathcal{N}_{v}$ denotes the neighbor node set of node $v$. We have the following 

\begin{myPro}
	\label{pro:mid_1}
	The low-frequency signal reduces the distance of the connected node while the high-frequency signal amplifies the distance of the connected node. And the mid-frequency signal has the least impact on the distance, i.e., the distance is slightly enlarged or stays unchanged.
\end{myPro}

To prove this proposition, we first introduce Lemma~\ref{lem:dt} and its proof is in the appendix.
\begin{myLem}
\label{lem:dt}
Given a node $u$, the following inequality holds:

\begin{equation}
    \mathbb{E}\left(\sum_{t}\frac{1}{d_{t}}\right) < 1 \,\text{with } t\in \mathcal{N}_{u} \,\text{and } t\in \mathcal{N}_{v}
\end{equation}
\end{myLem}
Thus, we can conclude from Equation~\ref{eq:dis} that $\|\mathcal{D}_{M} -\mathcal{D}  \|< \|\mathcal{D}_{L} -\mathcal{D}\| = \|\mathcal{D}_{H} -\mathcal{D}\|$. Furthermore, we can imagine that a low-pass filter is appropriate for graphs where neighbors of all nodes are of the same class as it can drive the representations of connected nodes to be more similar. A high-pass filter suits graphs where neighbor nodes are similar in feature but belong to different classes. In this case, high-pass filtering is required to increase the degree of discrimination. However, these properties are not related to our concerning robustness.
As for the mid-pass filter, one of its advantages is the stability of connected node distance regardless of the graph. We then provide a theory from the perspective of structural attacks.
\begin{myDef}
	The distance change rate of the connected node is defined as:
 \begin{equation}
	\label{eq:rate}
	\mathcal{R}(\Delta \mathcal{D}) = \left |\frac{\partial \mathcal{D}}{\partial d}\right |.
\end{equation}
\end{myDef}
And when the distance of the connected nodes is less affected by the structural attack, i.e., the distance change rate is lower, the model is supposed to be more robust. We have the following theorem:
\begin{mytheory}
\label{th:1}
	When the graph structure is perturbed, the distance change rate of the connected nodes via the mid-frequency signals is minimum compared with the one of low-frequency or high-frequency, which can 
 be denoted as $\mathcal{R}_{M}(\Delta \mathcal{D}) < \mathcal{R}_{H}(\Delta \mathcal{D}) = \mathcal{R}_{L}(\Delta \mathcal{D})$.
\end{mytheory}
\begin{proof}
For convenience, we let $d = d_{u}d_{v}$, and the distance change rate $\mathcal{R}(\Delta\mathcal{D})$ under each filter can be easily calculated according to Equation~\ref{eq:rate}. 
\noindent For the low-pass filter,
$$ \mathcal{R}_{L}(\Delta \mathcal{D})=\left |\frac{\partial \mathcal{D}_{L}}{\partial d}\right | = \left |\frac{\partial }{\partial d}(1-\frac{1}{\sqrt{d}})\mathcal{D}\right |
= \frac{1}{2}d^{-\frac{3}{2}} \mathcal{D}.$$
For the high-pass filter,
$$ \mathcal{R}_{H}(\Delta \mathcal{D})=\left |\frac{\partial\mathcal{D}_{H}}{\partial d} \right | = \left |\frac{\partial }{\partial d}(1+\frac{1}{\sqrt{d}})\mathcal{D}\right |
= \frac{1}{2}d^{-\frac{3}{2}}\mathcal{D}. $$
For the mid-pass filter,
\begin{equation}
    \begin{split}
        \mathcal{R}_{M}(\Delta \mathcal{D})&=\left |\frac{\partial \mathcal{D}_{M}}{\partial d}\right | \\
        &= 
\begin{cases}
	\frac{1}{2}d^{-\frac{3}{2}}\mathcal{D}\sum_{t}\frac{1}{d_{t}} & \text{if }   t\in \mathcal{N}_{u} \,\text{and } t\in \mathcal{N}_{v},\\ 
 0 & \text{otherwise.}
\end{cases} 
    \end{split}
\end{equation}

Since $\mathbb{E}\left(\sum_{t}\frac{1}{d_{t}}\right) < 1$, so we have $\mathcal{R}_{L}(\Delta \mathcal{D})= \mathcal{R}_{H}(\Delta \mathcal{D}) > \mathcal{R}_{M}(\Delta \mathcal{D})$. Then the theorem is proven. 
\end{proof}
\noindent\textbf{Remark1.} \emph{When the structure is disturbed, the mid-pass filter has the least effect on the distance of the connected nodes.}


\noindent\textbf{Observation.} Theory~\ref{th:1} proves that the mid-pass filter can maintain the robustness of distance between connected nodes under structural attacks. We then carry out experiments to verify this theory. We calculate the average distance change rate of connected nodes on real datasets. The evaluation metric is formulated as:
$$
\frac{\left | \Bar{\mathcal{D}_{0}} - \Bar{\mathcal{D}_{25}} \right |}{max(\Bar{\mathcal{D}_{0}} , \Bar{\mathcal{D}_{25})}}
$$
where $\Bar{\mathcal{D}_{i}}$ represents the average Euclidean distance of the connected nodes with a perturbation rate of $i\%$.
The result is shown in Table~\ref{tb:change_rate}. We can see that after the graph is perturbed, the distance change rate of connected nodes for Mid-GCN is only about half of that of GCN, which further verifies the robustness of the mid-frequency signals.

 \begin{table}[htbp]
 	\centering
 	\caption{Distance change rate of connected nodes whose degree is greater than 10 under a \emph{metattack} perturbation rate of 25\%.}
 	\label{tb:change_rate}
 	\begin{tabular}{ccc}
 		\toprule
 		\textbf{Datasets} & \textbf{GCN} & \textbf{Mid-GCN}\\
 		\midrule
 		Cora & 5.36\% & 2.66\% \\
 		Citeseer & 3.92\% & 2.40\% \\
 		Github & 6.49\% & 4.15\% \\
 		\bottomrule
 	\end{tabular}
 \end{table}


\subsection{Analysis of Structural Attack Influence on Spectral Domain}
\label{sec:theo2}
As graph filters directly impact signals on the spectral domain, we further investigate how perturbations over an edge can affect the spectral domain at low, mid and high frequencies. If an edge is inserted between disconnected node pair $(u,v)$, or the edge connecting $u$ and $v$ is deleted, the eigenvalues of graphs will be affected. Mathematically, we introduce Lemma~\ref{lem:per} to measure the differences in eigenvalues after structural attacks.

\begin{myLem}
	\label{lem:per}
	Given a graph $\mathcal{G} = (V, E)$ and one edge $e_{u,v}$ to be perturbed, the change of y-th eigenvalue of the normalized adjacency matrix $\widehat{A}$ after the perturbation is formulated as~\cite{LFR}

\begin{equation}
	\Delta \lambda_{y} = \begin{cases}
		2U_{y,u}\cdot U_{y,v} - \lambda_{y}(U^{2}_{y,u} + U^{2}_{y, v})  & \text{if } E\cup \{e_{u,v}\} \\ 
		-2U_{y,u}\cdot U_{y,v} + \lambda_{y}(U^{2}_{y,u} + U^{2}_{y, v}) & \text{if } E \setminus \{e_{u,v}\},
	\end{cases}
\end{equation}
where $U_{y, u}$ refers to the $u$-th element of eigenvector $U_{y}$ of $\widehat{A}$. Note that eigenvalues of $\widehat{A}$  discussed here range from -1 to 1.
\end{myLem}
Then we have the following theorem:
\begin{mytheory}
\label{th:2}
	Suppose $\lambda_{l}\in[p, 1], \lambda_{m}\in (-p, p), \lambda_{h}\in [-1,-p]$ is eigenvalue of low/mid/high-frequency signals, respectively, where $p \in (0, 1)$. Inequality $\left |\mathbb{E}( \Delta \lambda_{m} )\right | < \left |\mathbb{E}( \Delta \lambda_{h} )\right |$ as well as $\left |\mathbb{E}( \Delta \lambda_{m} )\right | < \left |\mathbb{E}( \Delta \lambda_{l}) \right |$ always holds.
\end{mytheory}





From Theorem~\ref{th:2} we can discover that the expectations of changes in the eigenvalue of mid-frequency signals are lower, indicating the robustness of mid-frequency signals under structural attacks from the perspective of spectral graph analysis. Moreover, we conduct experiments on some graph examples to verify this conclusion.

\noindent\textbf{Case study.} We provide some examples in Figure~\ref{fig:spectral_case} to demonstrate the changes in the graph signals within different frequency bands after attacks. It can be clearly observed that the graph signals within the mid-frequency band are more stable after attacks.

\begin{figure}[t]
	\centering
	\includegraphics[width=1.0\columnwidth]{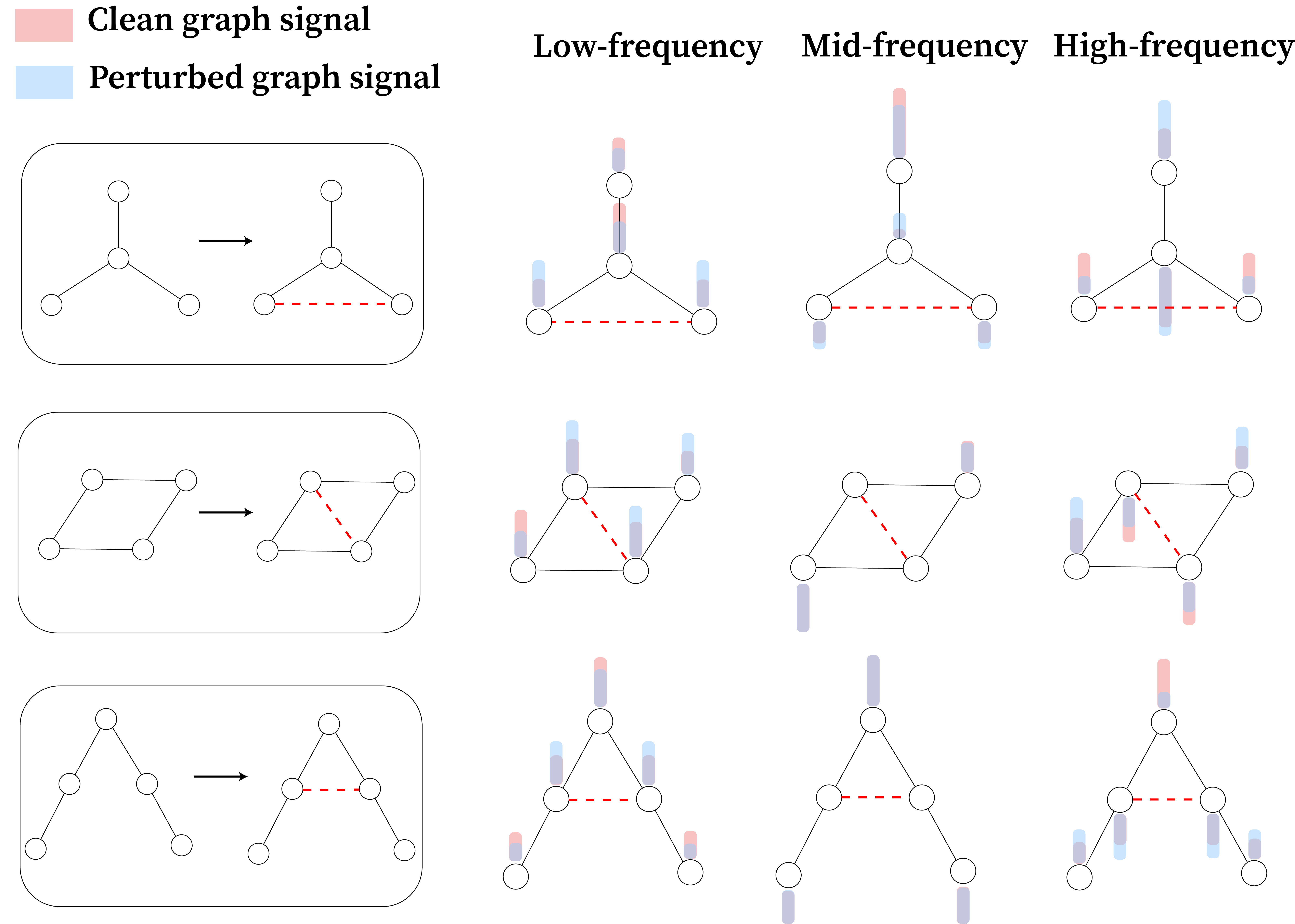}
	\caption{Various frequency bands under different structures, and the red dotted line represents the perturbed edge.}
	\label{fig:spectral_case}
\end{figure}

%

In addition, the mid-pass filter also has desirable properties in some cases; for example, when the structural attacks lead to changes in the homophily of the label, the mid-pass filter shows remarkable generalization ability varying from homophily to heterophily, and the proposition and proof are attached in Appendix~\ref{ap:gen}.

	\section{Experiment}
\label{sec:exp}
In this section, we compare the performance of Mid-GCN with other state-of-the-art models on the node classification task under a variety of graph attacks. Furthermore, we verify its effectiveness on graph data with different properties. 
We also conduct experiments to investigate the model efficiency and impacts of hyperparameter settings.
\subsection{Experiment Settings}
\subsubsection{Datasets.}
Following~\cite{prognn,rgcn}, we evaluate the proposed approach on three benchmark datasets, including two citation graphs (Cora and Citeseer), and one Github graph~\cite{Github}. 
We also select three additional datasets, including a graph without feature Polblogs~\cite{metattack}, a homophilious graph Cora-ML~\cite{cora-ml} and a heterophilious graph Film~\cite{filmdata2009} to show the capacity of Mid-GCN under various datasets. The statistics of datasets are organized in Table~\ref{tb:dataintroduce}. More details of datasets can be found in Appendix~\ref{ap:data-detail}. 
 \begin{table}[htbp]
 	\centering
 	\caption{Statistics of experimental datasets}
 	\label{tb:dataintroduce}
 	\begin{tabular}{cccccc}
 		\toprule
 		\textbf{Datasets} & \textbf{\#Nodes} & \textbf{\#Edges} &  \textbf{\#Features} & \textbf{\#Classes} \\
 		\midrule
 		Cora & 2,485 & 5,069 & 1,433 & 7 \\
 		Citeseer & 2,110 & 3,668 & 3,703 & 6 \\
 		Github & 3,150 & 71,310 & 4,005 & 2 \\
 		Polblogs & 1,222 & 16,714 & / & 2 \\
 		Cora-ML & 2,995 &  4,208 & 2,879 & 5 \\
 		Film & 7,600 & 33,544 & 931 & 5 \\
 		\bottomrule
 	\end{tabular}
 \end{table}

\begin{table*}[htbp]
\centering
\caption{Node classification accuracy under non-targeted attack \emph{Metattack}. \textbf{P(\%)} denotes the perturbation rate. For each result, we report the average performance over 10 runs and the corresponding standard deviation.}
\label{tb:meta}
\begin{tabular}{c|c|cccccccc}
    \toprule
    \textbf{Datasets} & \textbf{P(\%)} & \textbf{GCN} & \textbf{GAT} & \textbf{GCN-Jaccard} & \textbf{GCN-SVD} &
    \textbf{RGCN} &
    \textbf{Pro-GNN} & \textbf{GPRGNN} &
    \textbf{Mid-GCN} \\
    \midrule
    \multirow{6}*{Cora} & 0 & 83.50\scriptsize $\pm$0.44  & 83.97\scriptsize $\pm$0.69 & 82.05\scriptsize $\pm$0.51 & 80.63\scriptsize $\pm$0.45 & 83.09\scriptsize $\pm$0.44 & 83.42\scriptsize $\pm$0.52 & \textbf{85.39\scriptsize $\pm$0.81}& \underline{84.61\scriptsize $\pm$0.46} \\
    ~ & 5 & 76.55\scriptsize $\pm$0.79 & 80.44\scriptsize $\pm$0.74 & 79.13\scriptsize $\pm$0.59 & 78.93\scriptsize $\pm$0.54 & 77.42\scriptsize $\pm$0.39 & \underline{82.78\scriptsize $\pm$0.39} & 81.94\scriptsize $\pm$0.36 & \textbf{82.94\scriptsize $\pm$0.59}\\
    ~ & 10 & 70.39\scriptsize $\pm$1.28 & 75.61\scriptsize $\pm$0.59 &75.16\scriptsize $\pm$0.76 & 71.47\scriptsize $\pm$0.83 & 72.22\scriptsize $\pm$0.38 & \underline{79.03\scriptsize $\pm$0.59}  &76.41\scriptsize $\pm$1.35&  \textbf{80.14\scriptsize $\pm$0.86}\\
    ~ & 15 & 65.10\scriptsize $\pm$0.71 & 69.78\scriptsize $\pm$1.28 & 71.03\scriptsize $\pm$0.64& 66.69\scriptsize $\pm$1.18 & 66.82\scriptsize $\pm$0.39 & \underline{76.40\scriptsize $\pm$1.27} &72.47\scriptsize $\pm$0.47 & \textbf{77.77\scriptsize $\pm$0.75}\\
    ~ & 20 & 59.56\scriptsize $\pm$0.92 &59.54\scriptsize $\pm$0.92 & 65.71\scriptsize $\pm$0.89 &58.94\scriptsize $\pm$1.13 & 59.27\scriptsize $\pm$0.37 & \underline{73.32\scriptsize $\pm$1.56} &61.78\scriptsize $\pm$0.79& \textbf{76.58\scriptsize $\pm$0.29}\\
    ~ & 25 & 47.53\scriptsize $\pm$1.96 & 54.78\scriptsize $\pm$0.74& 60.82\scriptsize $\pm$1.08 & 52.06\scriptsize $\pm$1.19 & 50.51\scriptsize $\pm$0.78 & \underline{69.72\scriptsize $\pm$1.69} &57.16\scriptsize $\pm$1.25 & \textbf{72.89\scriptsize $\pm$0.81} \\
    \midrule
    \multirow{6}*{Citeseer} & 0 & 71.96\scriptsize $\pm$0.55 & 73.26\scriptsize $\pm$0.83 & 72.10\scriptsize $\pm$0.63 & 70.65\scriptsize $\pm$0.32 & 71.20\scriptsize $\pm$0.83 & 73.28\scriptsize $\pm$0.69 & \underline{73.52\scriptsize $\pm$0.23} & \textbf{74.17\scriptsize $\pm$0.28}\\
    ~ & 5 & 70.88\scriptsize $\pm$0.62 & 72.89\scriptsize $\pm$0.83 & 70.51\scriptsize $\pm$0.97 & 68.84\scriptsize $\pm$0.72 & 70.50\scriptsize $\pm$0.43 & \underline{73.09\scriptsize $\pm$0.34} & 72.79\scriptsize $\pm$0.18 & \textbf{74.31\scriptsize $\pm$0.42}\\
    ~ & 10 & 67.55\scriptsize $\pm$0.89 & 70.63\scriptsize $\pm$0.48 & 69.54\scriptsize $\pm$0.56 & 68.87\scriptsize $\pm$0.62 & 67.71\scriptsize $\pm$0.30 & \underline{72.51\scriptsize $\pm$0.75} &71.22\scriptsize $\pm$0.35 & \textbf{73.59\scriptsize $\pm$0.29}\\
    ~ & 15 & 64.52\scriptsize $\pm$1.11 & 69.02\scriptsize $\pm$1.09 & 65.95\scriptsize $\pm$0.94 & 63.26\scriptsize $\pm$0.96 & 65.69\scriptsize $\pm$0.37 & \underline{72.03\scriptsize $\pm$1.11} &69.19\scriptsize $\pm$0.39 & \textbf{73.69\scriptsize $\pm$0.29}\\
    ~ & 20 & 62.03\scriptsize $\pm$3.49 & 61.04\scriptsize $\pm$1.52 & 59.30\scriptsize $\pm$1.40 & 58.55\scriptsize $\pm$1.09 & 62.49\scriptsize $\pm$1.22 & \underline{70.02\scriptsize $\pm$2.28} & 63.10\scriptsize $\pm$0.50 & \textbf{71.51\scriptsize $\pm$0.83}\\
    ~ & 25& 56.94\scriptsize $\pm$2.09 & 61.85\scriptsize $\pm$1.12 & 59.89\scriptsize $\pm$1.47 & 57.18\scriptsize $\pm$1.87 & 55.35\scriptsize $\pm$0.66 & \underline{68.95\scriptsize $\pm$2.78} & 55.61\scriptsize $\pm$0.88& \textbf{69.12\scriptsize $\pm$0.72}\\
    \midrule
    \multirow{6}*{Github} & 0 & 72.92\scriptsize $\pm$0.13 & 72.81\scriptsize $\pm$0.12 & 72.93\scriptsize $\pm$0.56 & 73.31\scriptsize $\pm$1.15 & 73.16\scriptsize $\pm$0.19 & 73.34\scriptsize $\pm$0.84 & \underline{79.17\scriptsize $\pm$1.17} & \textbf{79.51\scriptsize $\pm$0.67}\\
    ~ & 5 & 72.81\scriptsize $\pm$0.07 & 72.43\scriptsize $\pm$1.13 & 71.85\scriptsize $\pm$2.18 & 72.91\scriptsize $\pm$0.12 & 73.09\scriptsize $\pm$0.31 & 72.89\scriptsize $\pm$0.07 &\underline{80.60\scriptsize $\pm$0.81} & \textbf{81.87\scriptsize $\pm$1.46}\\
    ~ & 10 & 72.61\scriptsize $\pm$0.57 & 72.97\scriptsize $\pm$0.13 & 72.63\scriptsize $\pm$0.98 & 72.78\scriptsize $\pm$0.09 & 73.06\scriptsize $\pm$0.16 & 72.75\scriptsize $\pm$0.18 &\underline{80.06\scriptsize $\pm$0.53} & \textbf{81.23\scriptsize $\pm$1.67}\\
    ~ & 15 & 72.97\scriptsize $\pm$0.11 & 72.97\scriptsize $\pm$0.06 & 72.79\scriptsize $\pm$0.51 & 72.97\scriptsize $\pm$1.13 & 73.22\scriptsize $\pm$0.21 &72.98\scriptsize $\pm$0.09 &\underline{80.28\scriptsize $\pm$0.51} & \textbf{80.48\scriptsize $\pm$0.25}\\
    ~ & 20 & 72.11\scriptsize $\pm$2.27 & 70.42\scriptsize $\pm$2.12 & 72.24\scriptsize $\pm$1.96 & 72.47\scriptsize $\pm$0.14 & 73.10\scriptsize $\pm$0.21 & 72.98\scriptsize $\pm$0.13 &\underline{80.19\scriptsize $\pm$0.51} & \textbf{81.08\scriptsize $\pm$0.96}\\
    ~ & 25 & 72.74\scriptsize $\pm$0.41 & 72.97\scriptsize $\pm$1.16 & 72.32\scriptsize $\pm$1.73 & 72.14\scriptsize $\pm$0.07 & 72.91\scriptsize $\pm$0.24 & 72.56\scriptsize $\pm$0.21 &\underline{79.86\scriptsize $\pm$0.87} & \textbf{80.37\scriptsize $\pm$1.51}\\
    \bottomrule
\end{tabular}
\end{table*}

\subsubsection{Baseline Methods.} To verify the effectiveness of Mid-GCN, we compare it with state-of-the-art GCNs that are listed as:

\begin{itemize}[leftmargin=*]
	\item \textbf{GCN}\cite{kipf2017gcn}: GCN is the most classic low-pass filtering GCNs, which updates node representations by aggregating neighbor information.
	
	\item \textbf{GAT}\cite{2018gat}: GAT is a graph neural network model that uses a local self-attention mechanism to generate node representations.
	
	\item \textbf{RGCN}\cite{rgcn}: RGCN represents the features of nodes as Gaussian distribution, which can improve model robustness by penalizing nodes with large variance.
	
	\item \textbf{GCN-SVD}\cite{gcn-svd}: GCN-SVD discovers that the nettack will affect lower singular values of the adjacency matrix and thus restores graphs by adopting low-rank approximation to preserve large singular values above a predefined threshold.   
	
	\item \textbf{GCN-Jaccard}\cite{gcn-jaccard}: GCN-Jaccard deletes adversarial edges that the Jaccard similarity of the node is below a certain threshold, and the graph can be denoised to be a clean graph. This clean graph is taken as the input of GCNs.
	
	\item \textbf{GPRGNN}\cite{GPRGNN2021}: GPRGNN combines Generalized PageRank and GNN to adaptively learn spectral graph filters.
	
	\item \textbf{Pro-GNN}\cite{prognn} Pro-GNN denoises graphs by constraining graph data to be low-rank, sparse, and feature smoothing. It is state-of-the-art defense GCN model. For each experiment, we only report the best results of its two variants (Pro-GNN-fs and Pro-GNN$^{3}$).
	
\end{itemize}

\subsubsection{Experiment Details.} 
We run experiments following the setup of ~\cite{prognn}. For each graph, we randomly choose 10\% of nodes for training, 10\% of nodes for validation and the remaining 80\% of nodes for testing except that the Film is divided according to~\cite{geomGCN}.
For GCN and GAT, we use the default parameters in the original paper. For RGCN, we tune the number of hidden units from \{16, 32, 64, 128\}. For GCN-Jaccard, the similarity threshold is chosen from 0.01 to 0.2. For GCN-SVD, the reduced rank of the perturbed graph is selected from \{5, 10, 15, 50, 100, 150, 200\}. For Pro-GNN, we follow the original settings to adjust the hyperparameters.
For Mid-GCN, we only tune the number of hidden layer dimension from \{64, 128\}, and the filter hyperparameter $\alpha$ from \{0.2, 0.3, 0.5, 055, 0.6, 2.0\}. The dropout rate is set to 0.6 and the learning rate is 0.01. We use Adam as the optimizer for model convergence with a weight decay rate of 5e-4. More reproducibility details are described in Appendix~\ref{ap:hp}.

\subsection{Experiment Results}
To examine the robustness of Mid-GCN, we design three types of attacks, i.e., non-targeted attack, targeted attack and feature attack:
\begin{itemize}[leftmargin=*]
	\item \textbf{Non-targeted attack.} Non-targeted attacks intend to attack the topology structure of the entire graphs, which reduces the overall performance of the GCN model. We adopt state-of-the-art non-targeted attack method \emph{metattack}~\cite{metattack} to perturbed graphs.
	
	\item \textbf{Targeted Attack.} Targeted attacks aim to attack target nodes and thus fool GCNs into misclassifying them, and only calculate the classification accuracy of these attacked nodes. We employ a representative targeted attack method, \emph{nettack}~\cite{nettack}.
	
	\item \textbf{Feature attack.} We are also interested in the performance under feature attacks as the attack on feature is another perturbation on graphs besides the topology attack~\cite{la-gcn,feature-attack2}. Since features of Cora, Citeser and Github only contain 0/1, we randomly flip 0/1 as the feature attack.
\end{itemize}

\begin{figure*}
	\centering
	\subfigure[Cora]{
		\begin{minipage}[t]{0.32\linewidth} 
			\centering
			\includegraphics[scale=0.14]{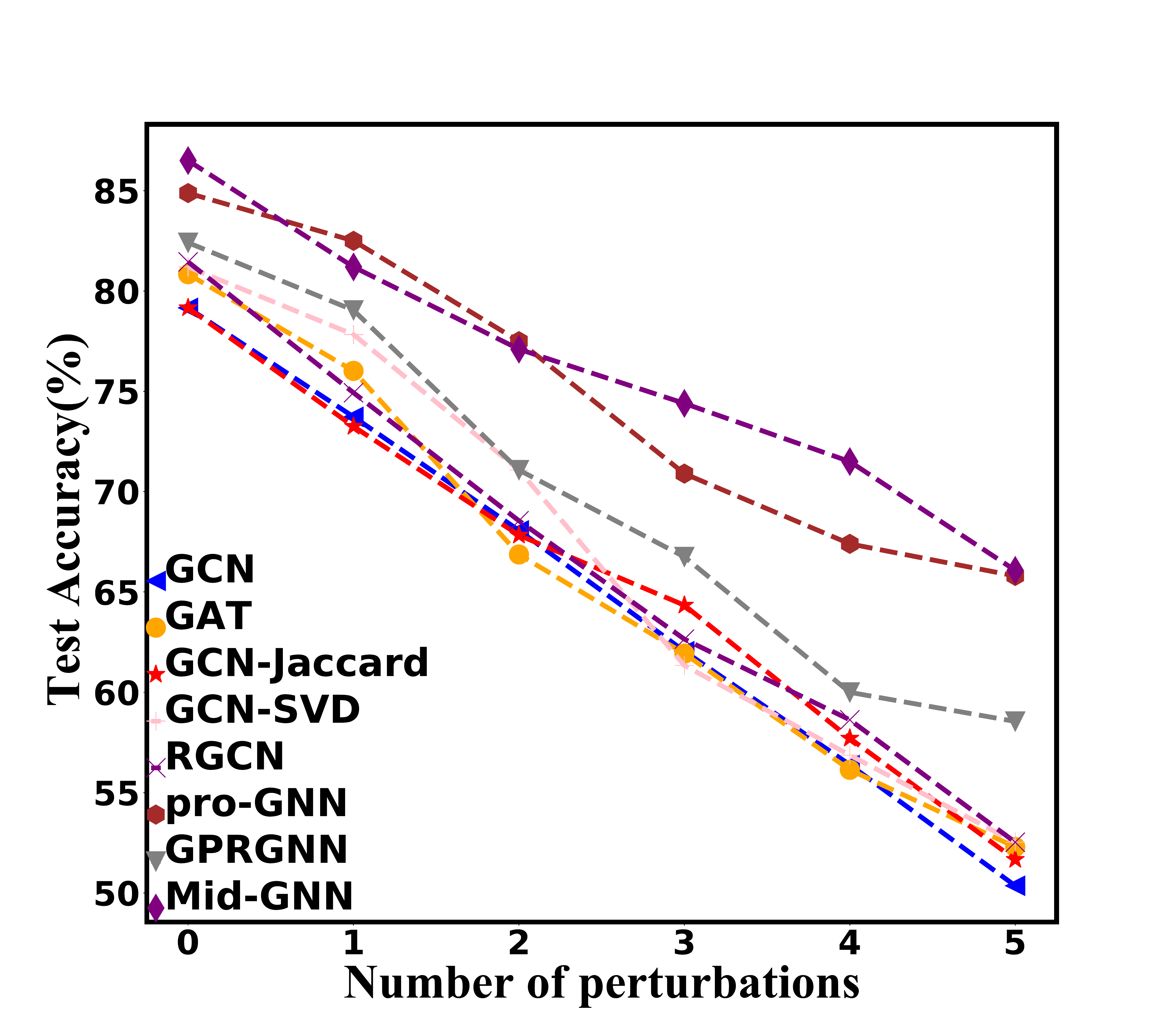} 
		\end{minipage}%
	}\subfigure[Citeseer]{
		\begin{minipage}[t]{0.32\linewidth}
			\centering
			\includegraphics[scale=0.14]{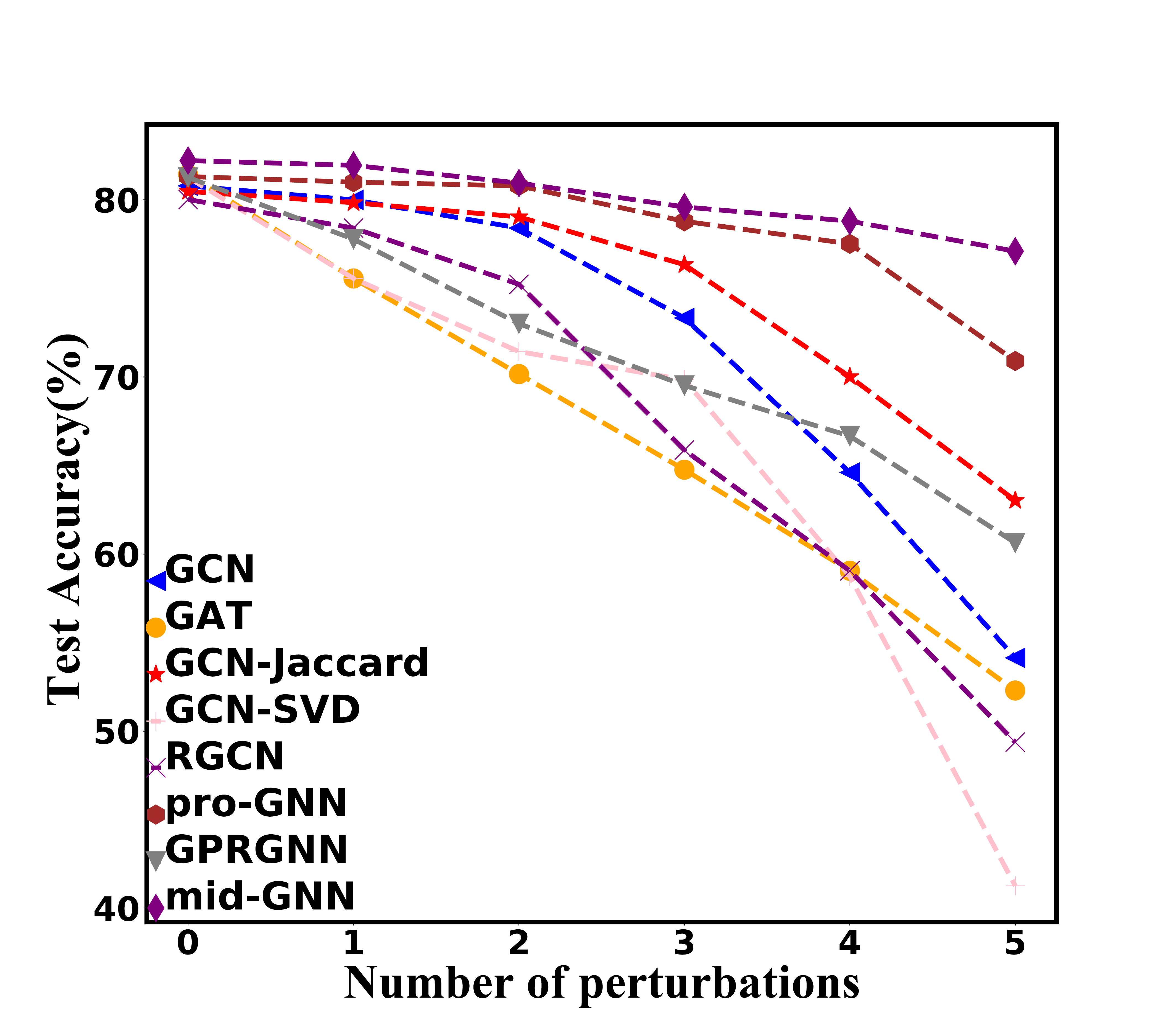}
		\end{minipage}%
	}\subfigure[Github]{
	\begin{minipage}[t]{0.32\linewidth}
		\centering
		\includegraphics[scale=0.14]{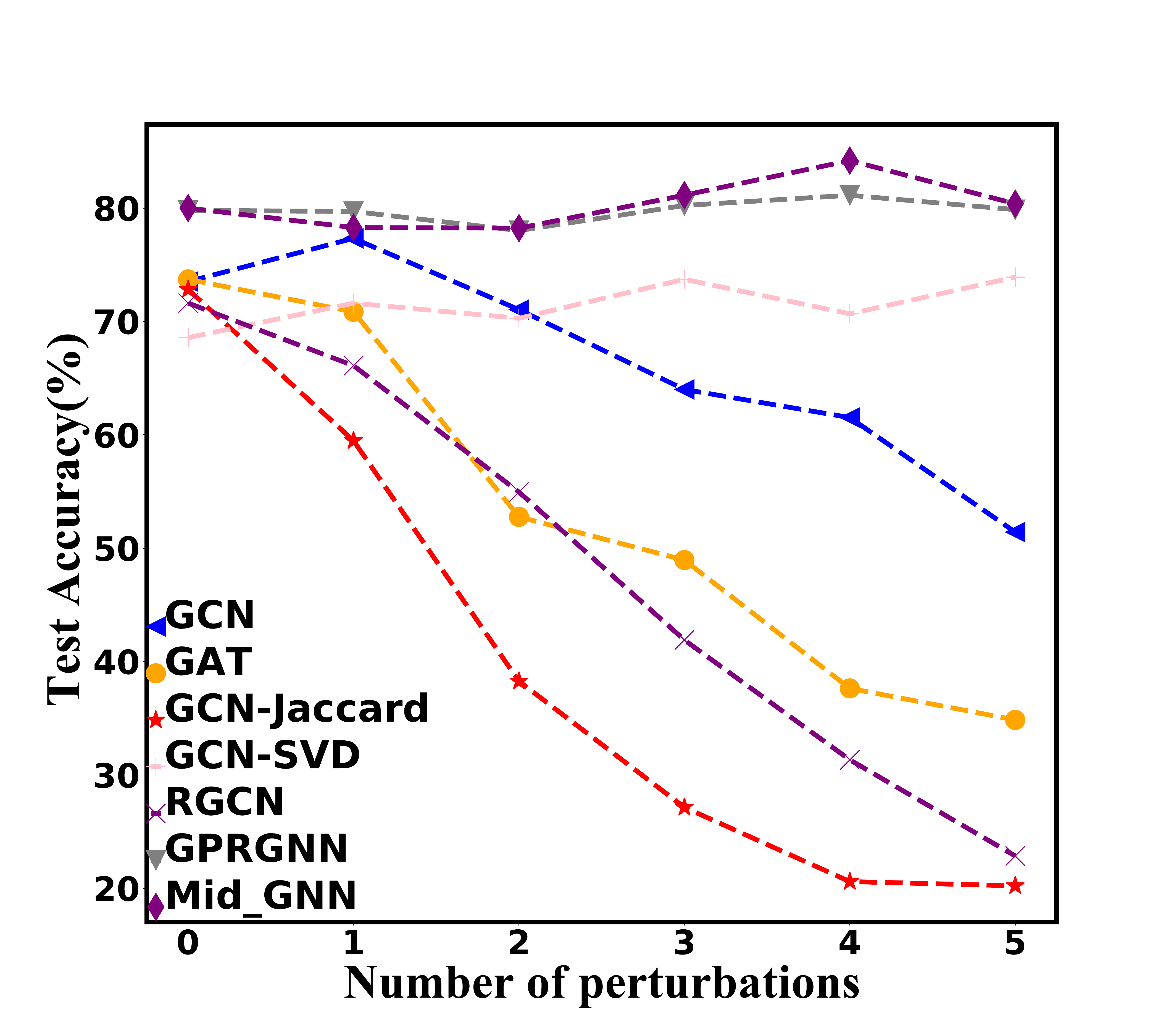}
	\end{minipage}%
}
	\caption{Node classification accuracy under targeted attack \emph{nettack}. Pro-GNN is not reported on Github as it times out against \emph{nettack}.
 }
	\label{fig:nettack}
\end{figure*}

\subsubsection{Against Non-targeted Adversarial Attacks}
We first exhibit the node classification accuracy of each model against non-targeted adversarial attacks under different perturbation rates ranging from 5\% to 25\%, and the vanilla model performance without perturbations is also shown. The default settings of \emph{metattack} are adopted to generate non-targeted adversarial attacks. 
The experimental results are organized in Table~\ref{tb:meta}. Value in bold indicates the best performance and the underlined value is the suboptimal one. Our observations are as follows:

\begin{itemize}[leftmargin=*]
\item Under metattack, our method achieves the best results at all perturbation rates and obtains considerable improvements over the representative low-pass filtering model GCN. Mid-GCN improves the performance of GCN by 18\% on Cora, 10.8\% on Citeseer and 10.7\% on Github under various perturbation rates on average. Comparing SOTA model Pro-GNN, our method also has appreciable improvements under various perturbation rates, with an average increase of 2.1\% in Cora, 1.7\% in Citeseer, and 10.4\% in Github.
In particular, under 25\% perturbation on the three datasets, our model improves the performance over GCN by a margin of 25\%, 13\% and 8\%, respectively.
\item There are many critical assumptions and constraints on data when applying Pro-GNN, such as low-rank, feature smoothing and sparseness, resulting in the poor performance on a dense and high-rank graph Github. GCN-SVD, which is mainly designed for targeted attacks, is also of great difficulty in restoring the graph structure of Github with a large number of global edge perturbations.
GPRGNN that adaptively learns graph filters shows promising achievements on clean graphs without perturbations; however, it can not handle attacks on spare graphs like Cora and Citeseer and the accuracy drops sharply when the perturbation rate increases.
\item An impressive phenomenon that can be seen from results on the Github dataset is that most of the models are less affected by adversarial attacks and is capable of obtaining stable results when the perturbation rate varies. Our method even achieves the best performance with a perturbation rate of 5\% compared with the vanilla one. The main reason is that the Github graph is much dense than other graphs (the average node degree is around 45), and \emph{metattack} tends to add edges~\cite{metattack} which has fewer impacts on a dense graph rather than a sparse graph.
\end{itemize}

\subsubsection{Against Targeted Adversarial Attack.} 
In the targeted attack experiment, we adopt the representative \emph{nettack} method with default settings of original paper~\cite{nettack,prognn}, i.e., we impose perturbations on neighbor nodes of each target node with a step size of 1 and the number of affected nodes is 1 $\sim$ 5. Following~\cite{prognn}, we select nodes whose degree is greater than 10 in the test set as target nodes. Since the Github graph is relatively dense so we sample 8\% of the nodes for attacks. Experimental results are depicted in Figure~\ref{fig:nettack}. We can observe that the mid-pass filtering GCN can always maintain the best performance when the number of targeted attacks increases. There are significant improvements over the low-pass filtering GCN and it becomes larger when the attacks are more severe. 

\subsubsection{Against Feature Attack.}
Apart from structural perturbations on graphs, feature attack is also a critical attack perspective as node features are extensively used by GCNs. We follow~\cite{la-gcn} to randomly flip  0/1 value on each dimension of Cora and Citeseer. We design two experiment settings, i.e., normalized features and non-normalized features after the attacks as the normalization will affect all the dimension of feature values.
Model performance is displayed in Figure~\ref{fig:fea_att} from which we have the following observations:

\begin{itemize}[leftmargin=*]
\item Mid-GCN achieves desirable performance regardless of whether the features are normalized or not, which demonstrates the robustness of Mid-GCN under the feature attacks.
\item Mid-GCN and RGCN are both consistently benefited from the normalization operation, but the impact on other models varies, and it even causes destructive damages to some of them.
For example, on Cora, when the number of perturbations is 70, classification accuracy of GCN, GCN-Jaccard and Pro-GNN with normalized features decreases by 10\%, 38\% and 21\% compared with the non-normalized one, respectively. 
\end{itemize}

\begin{figure}
	\centering
	\subfigure[Normalized Cora feature.]{
		\begin{minipage}[t]{0.55\linewidth} 
			\centering
			\includegraphics[scale=0.11]{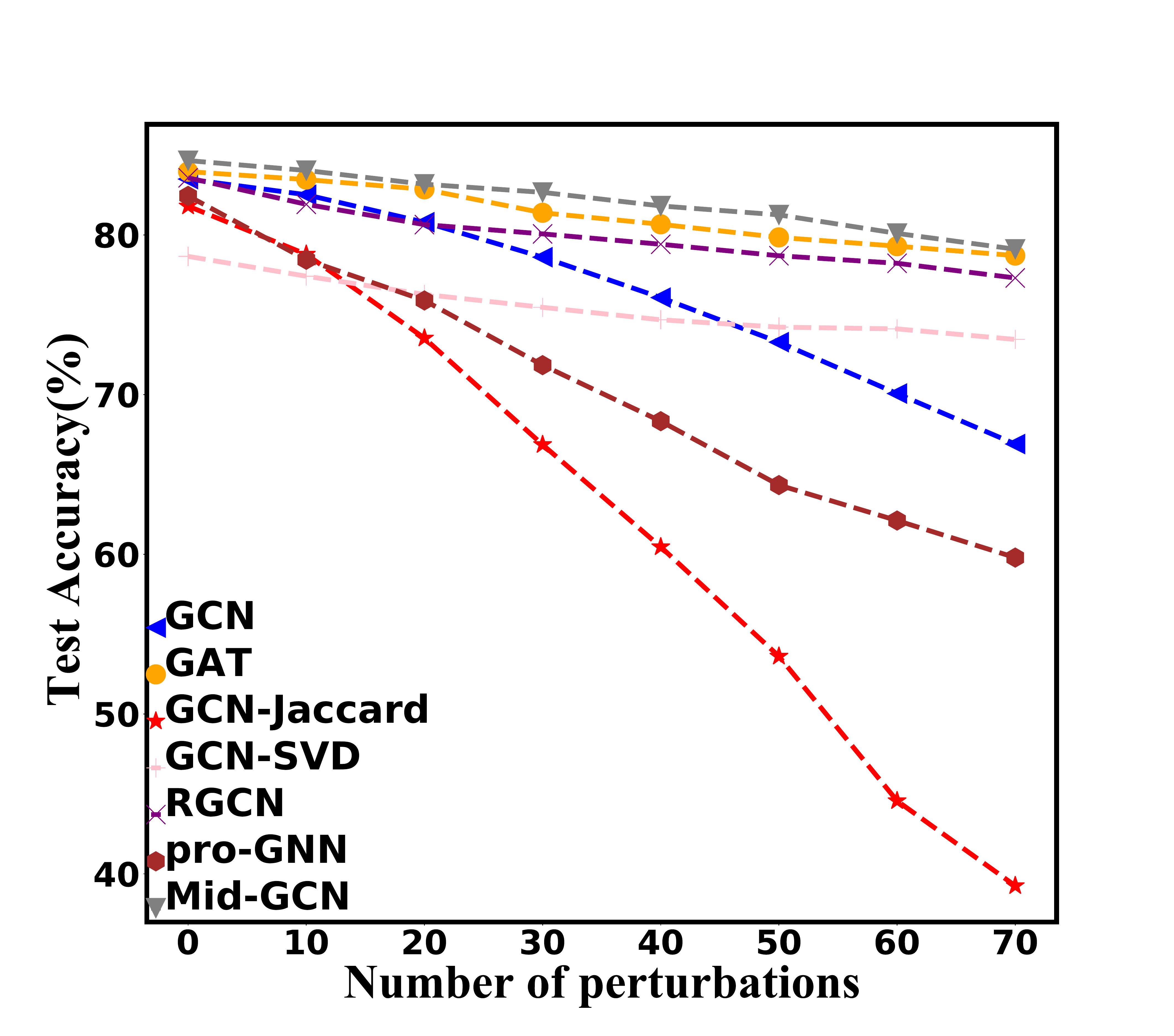} 
		\end{minipage}%
	}\subfigure[Non-normalized Cora feature.]{
		\begin{minipage}[t]{0.55\linewidth}
			\centering
			\includegraphics[scale=0.11]{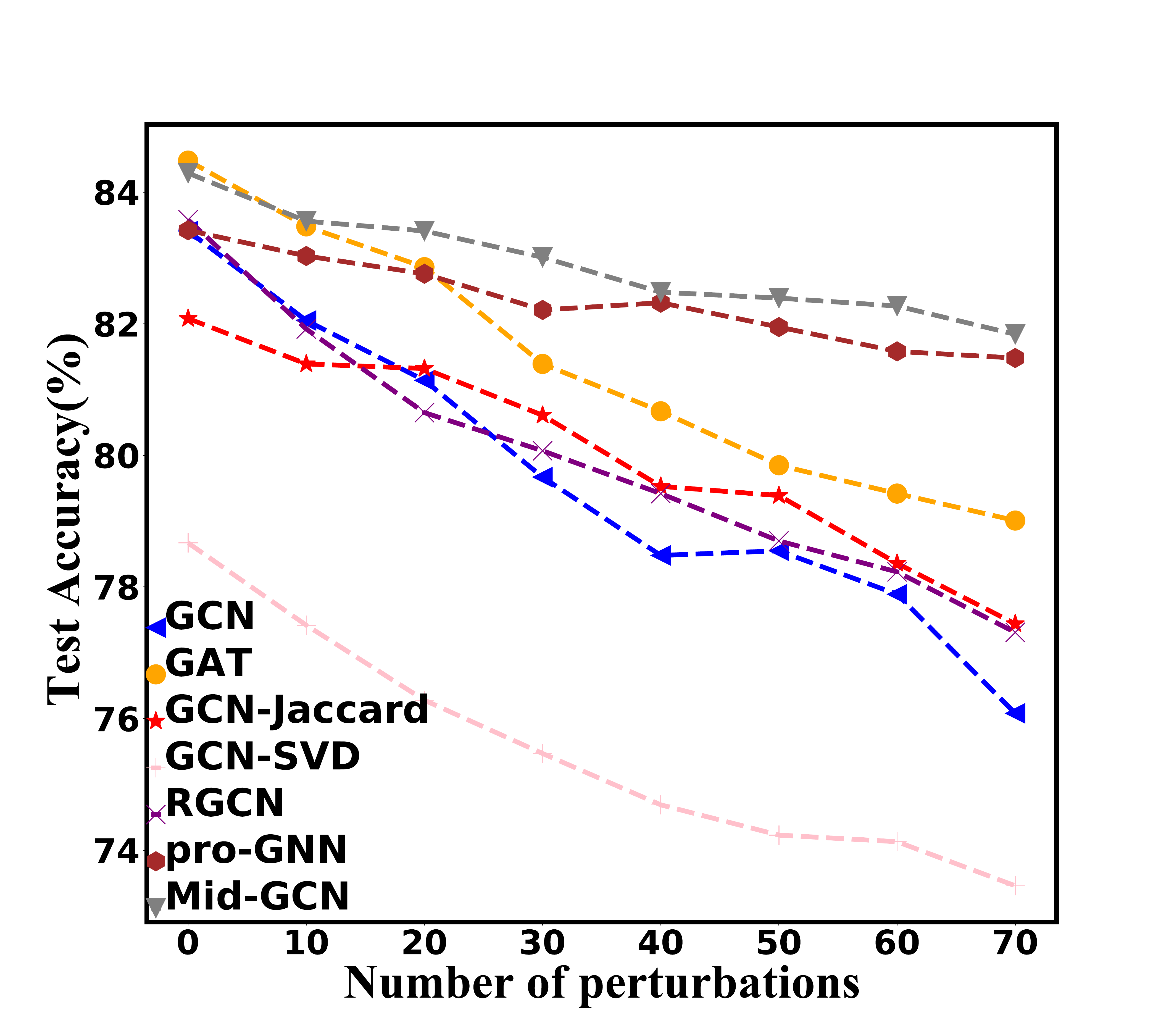}
		\end{minipage}%
	}
\subfigure[Normalized Citeseer feature.]{
		\begin{minipage}[t]{0.55\linewidth}
			\centering
			\includegraphics[scale=0.11]{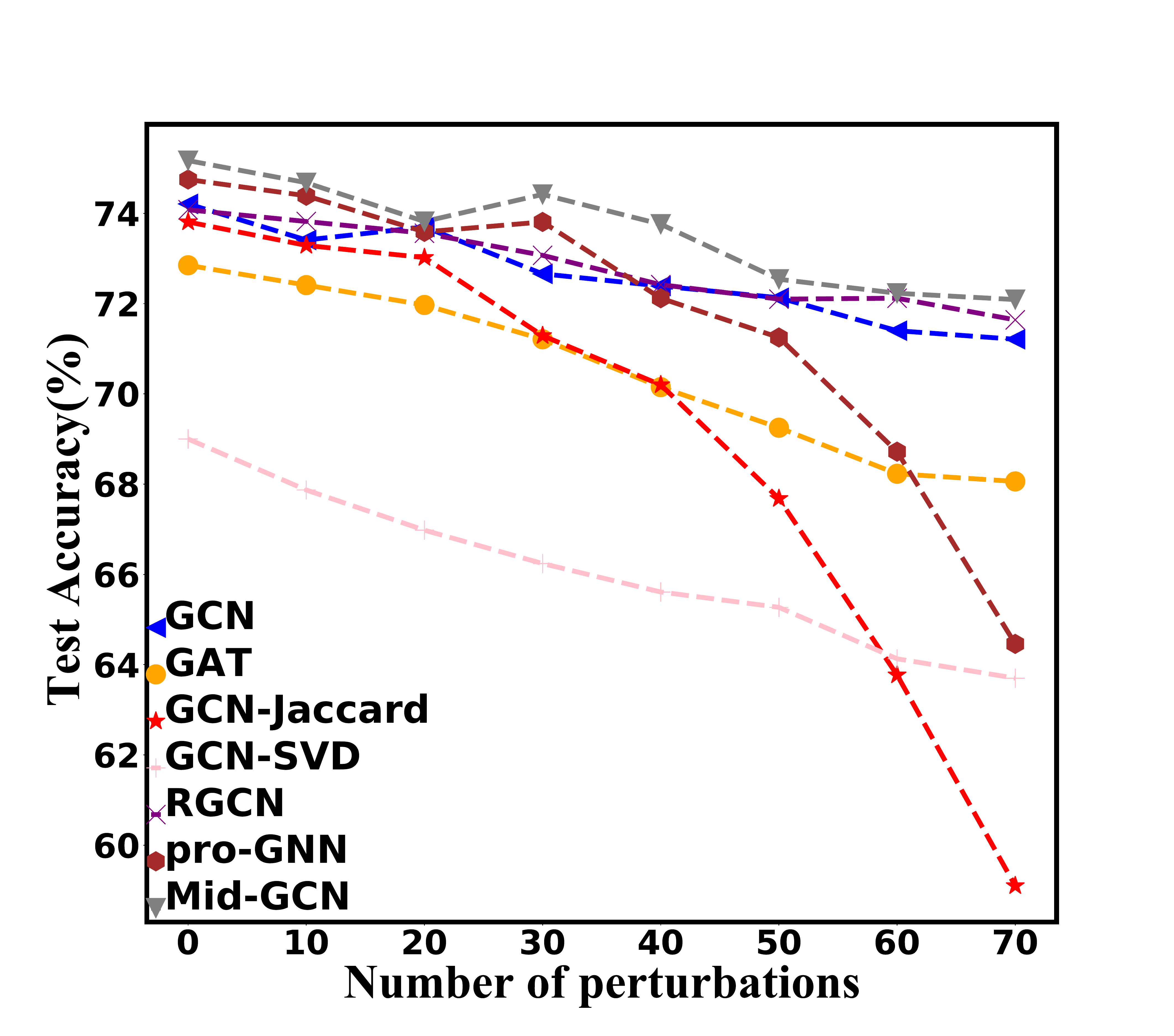}
		\end{minipage}%
	}\subfigure[Non-normalized Citeseer feature.]{
	\begin{minipage}[t]{0.55\linewidth}
		\centering
		\includegraphics[scale=0.11]{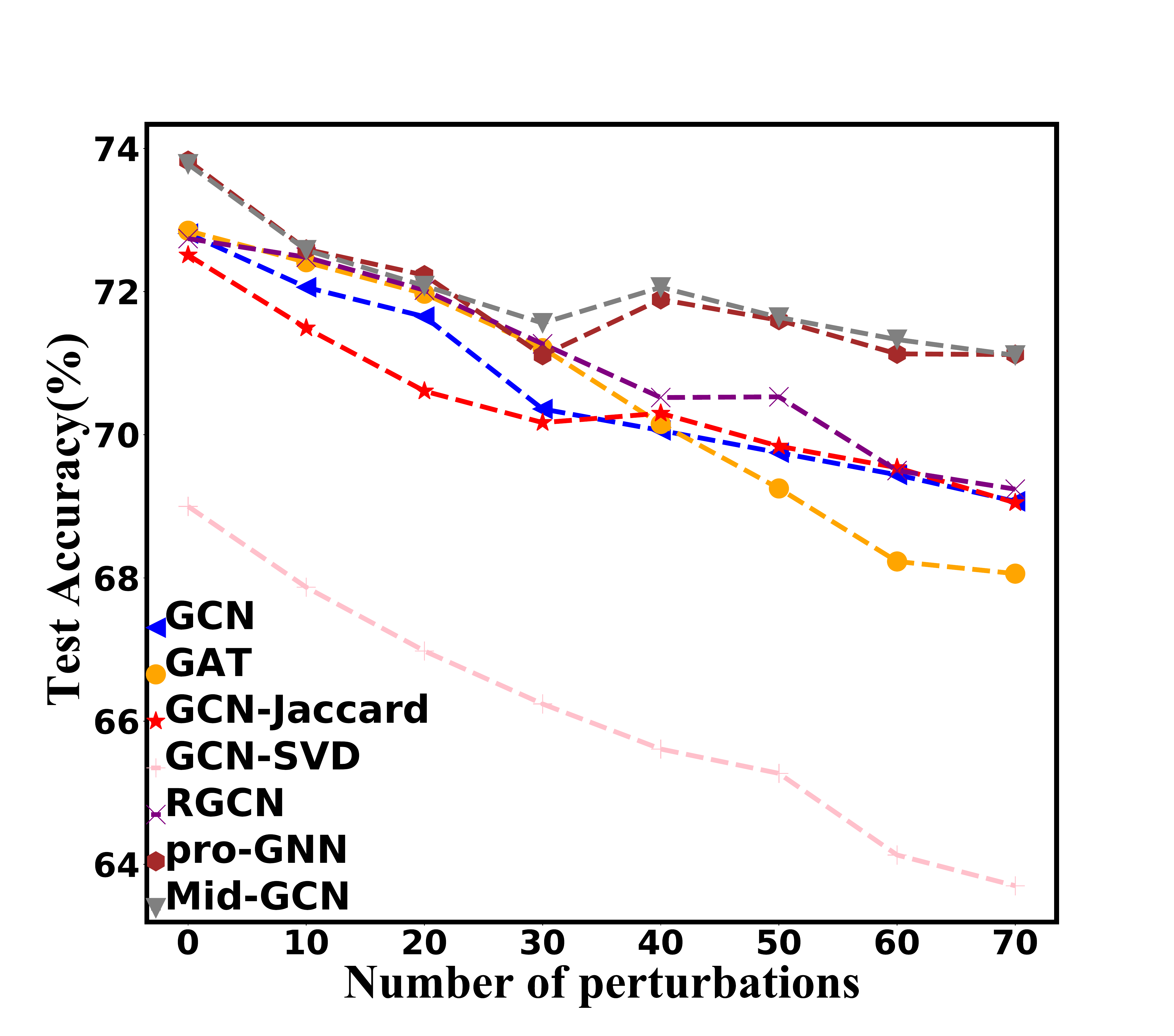}
	\end{minipage}%
}
	\caption{Results of different models under features attack.}
	\label{fig:fea_att}
\end{figure}


\subsection{Evaluation on Various Graphs.}
\label{sec:wide_data}
In this section, we study the performance when the graph is heavily poisoned on more diverse datasets. Specifically, we attack the graph data with 25\% perturbation rate by \emph{metattack}. Note that Polblogs is a dataset without node features, whose input only contains the graph structure. Cora-ML is a homophily graph while Film is a heterophily graph. From Table~\ref{tb:wide_data} we can discover that Mid-GCN can always perform well even with severe poisoning on different categories of graph data, which provides powerful support for the promotion of Mid-GCN.

 \begin{table}[htbp]
 	\centering
 	\caption{Classification results on three different graphs under 25\% perturbation rate \emph{metattack}.
}
 	\label{tb:wide_data}
  \resizebox{1.0\columnwidth}{!}{%
 	\begin{tabular}{c|cc|cc|cc}
 		\toprule
 		\multirow{2}{*}{Models} & \multicolumn{2}{c}{Polblogs}  &   \multicolumn{2}{c}{Cora-ML}  &   \multicolumn{2}{c}{Film}    \\
 		~ & Clean & Attacked & Clean & Attacked & Clean & Attacked \\
 		\midrule
 		
 		GCN & 95.69\scriptsize $\pm$0.38 & 45.23\scriptsize $\pm$1.36 & 85.85\scriptsize $\pm$0.30 & 48.80\scriptsize $\pm$0.91 & 27.82\scriptsize $\pm$1.10 & 25.36\scriptsize $\pm$0.13  \\
 		Pro-GNN & 93.20\scriptsize $\pm$0.64 & 63.18\scriptsize $\pm$4.40 & 85.38\scriptsize $\pm$0.14  & 65.52\scriptsize $\pm$0.23 & 30.08\scriptsize $\pm$0.96 & 27.03\scriptsize $\pm$0.61  \\
 		Mid-GCN & \textbf{95.80\scriptsize $\pm$0.26} & \textbf{64.66\scriptsize $\pm$0.86} & \textbf{86.56\scriptsize $\pm$0.28} & \textbf{67.18\scriptsize $\pm$1.35} & \textbf{34.32\scriptsize $\pm$0.84} & \textbf{33.17\scriptsize $\pm$0.70}  \\
 		\bottomrule
 	\end{tabular}
  }
 \end{table}

\subsection{Time Complexity Study.}
Time complexity is a vital indicator of model performance.
Therefore, we illustrate both the effectiveness and efficiency of GCN, GAT, Pro-GNN and Mid-GCN on Figure~\ref{fig:time} for a comprehensive understanding. 
Pointers near the top left corner indicate better performance both in effectiveness and efficiency.
It can be found that Mid-GCN obtains an excellent balance between accuracy and running time.
\begin{figure}[t]
	\centering
	\includegraphics[width=0.7\columnwidth]{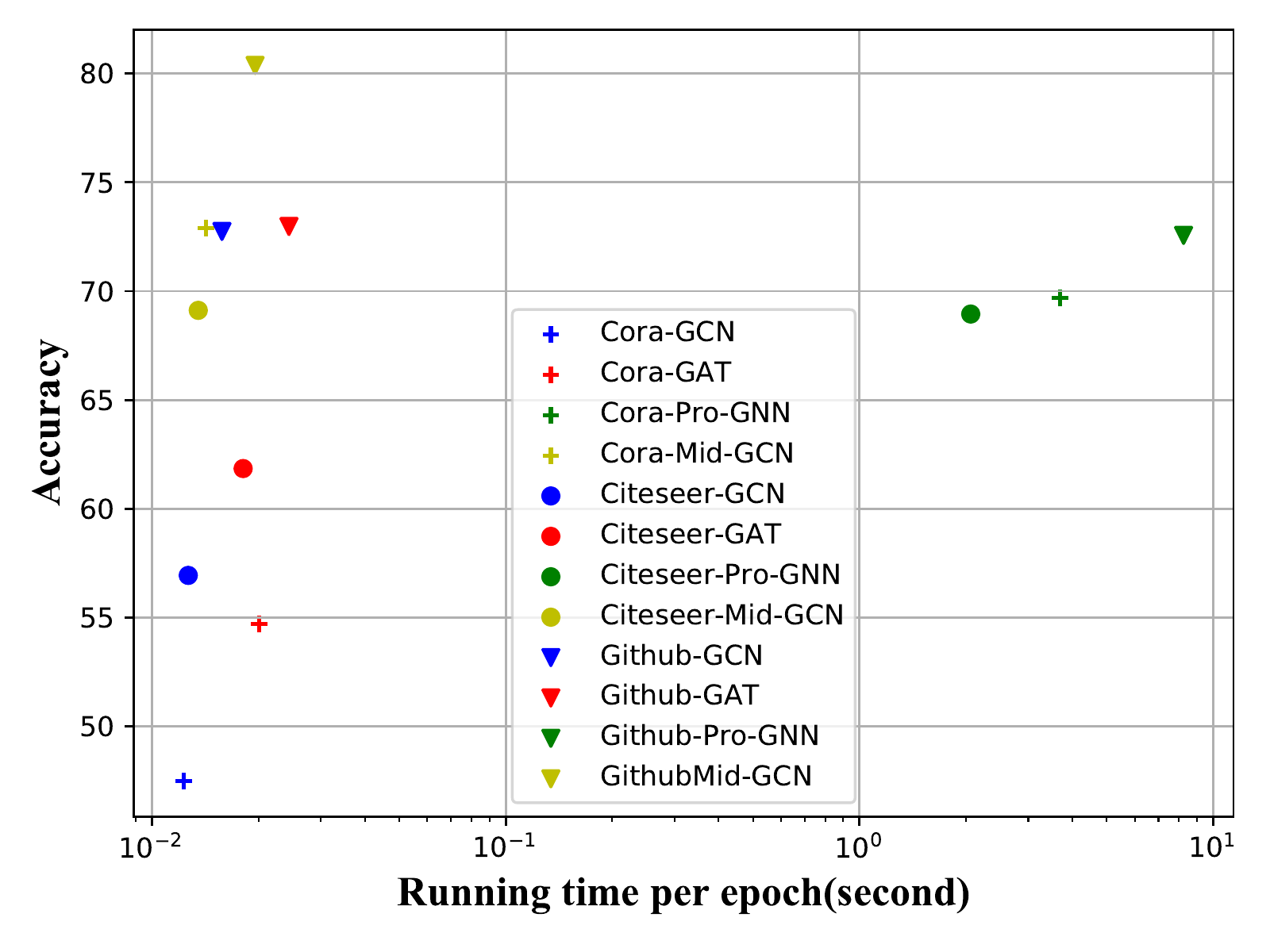} 
	\caption{Accuracy v.s. running time.}
	\label{fig:time}
\end{figure}

\subsection{Hyperparameter Study.}

We also explore the impact of hyperparameters $\alpha$ on the model prediction accuracy. $\alpha$ varies from 0 to 2, and the results on Cora and Citeseer are shown in Fig.~\ref{fig:hyper} under a perturbation
rate of 10\% of metattack. We can see that the accuracy remains stable when $\alpha$ ranges from 0 to 0.7 where the filter can be treated as a mid-high-pass filter. In other words, mid-frequency and high-frequency signals play a critical role in defending against generative adversarial attacks, and it is consistent with our previous observations.
\begin{figure}[t]
	\centering
	\includegraphics[width=0.7\columnwidth]{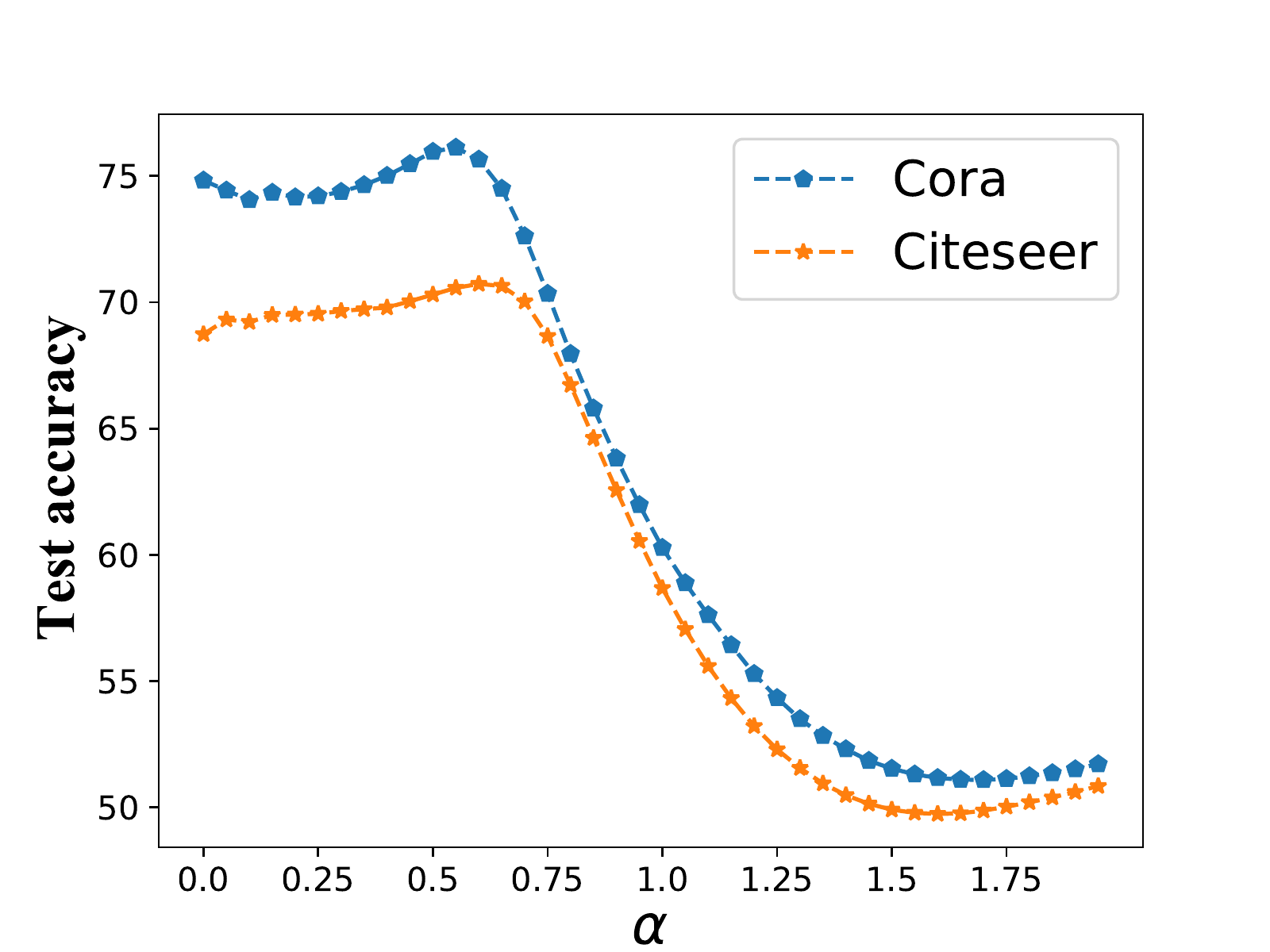}
	\caption{Hyperparameter study on Cora and Citeseer.}
	\label{fig:hyper}
\end{figure}

	\section{Conclusion}
In this paper, we discovered the importance of mid-frequency signals on GCNs that are less studied especially from the perspective of robustness. We propose a mid-pass filtering GCN model, \emph{Mid-GCN}, and conduct an in-depth theoretical analysis to prove its effectivenss over attacks compared with low-pass and high-pass GCN.  
Our experiments further verify that Mid-GCN consistently surpasses state-of-the-art baselines and improves robustness under various adversarial attacks and graphs with distinct properties. The mid-pass filtering GCN is believed to be effective in a broader application prospect and will be the next step of future exploration.


\clearpage
\bibliographystyle{ACM-Reference-Format}
\bibliography{reference}


\begin{thebibliography}{47}


\ifx \showCODEN    \undefined \def \showCODEN     #1{\unskip}     \fi
\ifx \showDOI      \undefined \def \showDOI       #1{#1}\fi
\ifx \showISBNx    \undefined \def \showISBNx     #1{\unskip}     \fi
\ifx \showISBNxiii \undefined \def \showISBNxiii  #1{\unskip}     \fi
\ifx \showISSN     \undefined \def \showISSN      #1{\unskip}     \fi
\ifx \showLCCN     \undefined \def \showLCCN      #1{\unskip}     \fi
\ifx \shownote     \undefined \def \shownote      #1{#1}          \fi
\ifx \showarticletitle \undefined \def \showarticletitle #1{#1}   \fi
\ifx \showURL      \undefined \def \showURL       {\relax}        \fi
\providecommand\bibfield[2]{#2}
\providecommand\bibinfo[2]{#2}
\providecommand\natexlab[1]{#1}
\providecommand\showeprint[2][]{arXiv:#2}

\bibitem[Bo et~al\mbox{.}(2021)]%
        {fagcn}
\bibfield{author}{\bibinfo{person}{Deyu Bo}, \bibinfo{person}{Xiao Wang},
  \bibinfo{person}{Chuan Shi}, {and} \bibinfo{person}{Huawei Shen}.}
  \bibinfo{year}{2021}\natexlab{}.
\newblock \showarticletitle{Beyond Low-frequency Information in Graph
  Convolutional Networks}. In \bibinfo{booktitle}{\emph{Thirty-Fifth {AAAI}
  Conference on Artificial Intelligence, {AAAI}2021}}.
  \bibinfo{publisher}{{AAAI} Press}, \bibinfo{pages}{3950--3957}.
\newblock


\bibitem[Bojchevski and G{\"u}nnemann(2017)]%
        {cora-ml}
\bibfield{author}{\bibinfo{person}{Aleksandar Bojchevski} {and}
  \bibinfo{person}{Stephan G{\"u}nnemann}.} \bibinfo{year}{2017}\natexlab{}.
\newblock \showarticletitle{Deep gaussian embedding of graphs: Unsupervised
  inductive learning via ranking}.
\newblock \bibinfo{journal}{\emph{arXiv preprint arXiv:1707.03815}}
  (\bibinfo{year}{2017}).
\newblock


\bibitem[Bruna et~al\mbox{.}(2013)]%
        {spectral_lecun}
\bibfield{author}{\bibinfo{person}{Joan Bruna}, \bibinfo{person}{Wojciech
  Zaremba}, \bibinfo{person}{Arthur Szlam}, {and} \bibinfo{person}{Yann
  LeCun}.} \bibinfo{year}{2013}\natexlab{}.
\newblock \showarticletitle{Spectral networks and locally connected networks on
  graphs}.
\newblock \bibinfo{journal}{\emph{arXiv preprint arXiv:1312.6203}}
  (\bibinfo{year}{2013}).
\newblock


\bibitem[Chang et~al\mbox{.}(2021)]%
        {LFR}
\bibfield{author}{\bibinfo{person}{Heng Chang}, \bibinfo{person}{Yu Rong},
  \bibinfo{person}{Tingyang Xu}, \bibinfo{person}{Yatao Bian},
  \bibinfo{person}{Shiji Zhou}, \bibinfo{person}{Xin Wang},
  \bibinfo{person}{Junzhou Huang}, {and} \bibinfo{person}{Wenwu Zhu}.}
  \bibinfo{year}{2021}\natexlab{}.
\newblock \showarticletitle{Not all low-pass filters are robust in graph
  convolutional networks}.
\newblock \bibinfo{journal}{\emph{Advances in Neural Information Processing
  Systems}}  \bibinfo{volume}{34} (\bibinfo{year}{2021}),
  \bibinfo{pages}{25058--25071}.
\newblock


\bibitem[Chen et~al\mbox{.}(2022)]%
        {matengfei}
\bibfield{author}{\bibinfo{person}{Zhixian Chen}, \bibinfo{person}{Tengfei Ma},
  {and} \bibinfo{person}{Yang Wang}.} \bibinfo{year}{2022}\natexlab{}.
\newblock \showarticletitle{When Does A Spectral Graph Neural Network Fail in
  Node Classification?}
\newblock \bibinfo{journal}{\emph{arXiv preprint arXiv:2202.07902}}
  (\bibinfo{year}{2022}).
\newblock


\bibitem[Chien et~al\mbox{.}(2021)]%
        {GPRGNN2021}
\bibfield{author}{\bibinfo{person}{Eli Chien}, \bibinfo{person}{Jianhao Peng},
  \bibinfo{person}{Pan Li}, {and} \bibinfo{person}{Olgica Milenkovic}.}
  \bibinfo{year}{2021}\natexlab{}.
\newblock \showarticletitle{Adaptive Universal Generalized PageRank Graph
  Neural Network}. In \bibinfo{booktitle}{\emph{9th International Conference on
  Learning Representations, {ICLR} 2021, Virtual Event, Austria, May 3-7,
  2021}}.
\newblock


\bibitem[Chung and Graham(1997)]%
        {spectral1997}
\bibfield{author}{\bibinfo{person}{Fan~RK Chung} {and}
  \bibinfo{person}{Fan~Chung Graham}.} \bibinfo{year}{1997}\natexlab{}.
\newblock \bibinfo{booktitle}{\emph{Spectral graph theory}}.
  Vol.~\bibinfo{volume}{92}.
\newblock \bibinfo{publisher}{American Mathematical Soc.}
\newblock


\bibitem[Coley et~al\mbox{.}(2019)]%
        {molecular2}
\bibfield{author}{\bibinfo{person}{Connor~W Coley}, \bibinfo{person}{Wengong
  Jin}, \bibinfo{person}{Luke Rogers}, \bibinfo{person}{Timothy~F Jamison},
  \bibinfo{person}{Tommi~S Jaakkola}, \bibinfo{person}{William~H Green},
  \bibinfo{person}{Regina Barzilay}, {and} \bibinfo{person}{Klavs~F Jensen}.}
  \bibinfo{year}{2019}\natexlab{}.
\newblock \showarticletitle{A graph-convolutional neural network model for the
  prediction of chemical reactivity}.
\newblock \bibinfo{journal}{\emph{Chemical science}} \bibinfo{volume}{10},
  \bibinfo{number}{2} (\bibinfo{year}{2019}), \bibinfo{pages}{370--377}.
\newblock


\bibitem[Dai et~al\mbox{.}(2018)]%
        {attack_2018}
\bibfield{author}{\bibinfo{person}{Hanjun Dai}, \bibinfo{person}{Hui Li},
  \bibinfo{person}{Tian Tian}, \bibinfo{person}{Xin Huang},
  \bibinfo{person}{Lin Wang}, \bibinfo{person}{Jun Zhu}, {and}
  \bibinfo{person}{Le Song}.} \bibinfo{year}{2018}\natexlab{}.
\newblock \showarticletitle{Adversarial attack on graph structured data}. In
  \bibinfo{booktitle}{\emph{International conference on machine learning}}.
  PMLR, \bibinfo{pages}{1115--1124}.
\newblock


\bibitem[Defferrard et~al\mbox{.}(2016)]%
        {cheyGCN2016}
\bibfield{author}{\bibinfo{person}{Micha{\"e}l Defferrard},
  \bibinfo{person}{Xavier Bresson}, {and} \bibinfo{person}{Pierre
  Vandergheynst}.} \bibinfo{year}{2016}\natexlab{}.
\newblock \showarticletitle{Convolutional neural networks on graphs with fast
  localized spectral filtering}.
\newblock \bibinfo{journal}{\emph{Advances in neural information processing
  systems}}  \bibinfo{volume}{29} (\bibinfo{year}{2016}),
  \bibinfo{pages}{3844--3852}.
\newblock


\bibitem[Entezari et~al\mbox{.}(2020)]%
        {gcn-svd}
\bibfield{author}{\bibinfo{person}{Negin Entezari}, \bibinfo{person}{Saba~A
  Al-Sayouri}, \bibinfo{person}{Amirali Darvishzadeh}, {and}
  \bibinfo{person}{Evangelos~E Papalexakis}.} \bibinfo{year}{2020}\natexlab{}.
\newblock \showarticletitle{All you need is low (rank) defending against
  adversarial attacks on graphs}. In \bibinfo{booktitle}{\emph{Proceedings of
  the 13th International Conference on Web Search and Data Mining}}.
  \bibinfo{pages}{169--177}.
\newblock


\bibitem[Han et~al\mbox{.}(2019)]%
        {graph_web2}
\bibfield{author}{\bibinfo{person}{Fred~X Han}, \bibinfo{person}{Di Niu},
  \bibinfo{person}{Kunfeng Lai}, \bibinfo{person}{Weidong Guo},
  \bibinfo{person}{Yancheng He}, {and} \bibinfo{person}{Yu Xu}.}
  \bibinfo{year}{2019}\natexlab{}.
\newblock \showarticletitle{Inferring search queries from web documents via a
  graph-augmented sequence to attention network}. In
  \bibinfo{booktitle}{\emph{The World Wide Web Conference}}.
  \bibinfo{pages}{2792--2798}.
\newblock


\bibitem[He et~al\mbox{.}(2021)]%
        {bernnet}
\bibfield{author}{\bibinfo{person}{Mingguo He}, \bibinfo{person}{Zhewei Wei},
  \bibinfo{person}{Hongteng Xu}, {et~al\mbox{.}}}
  \bibinfo{year}{2021}\natexlab{}.
\newblock \showarticletitle{Bernnet: Learning arbitrary graph spectral filters
  via bernstein approximation}.
\newblock \bibinfo{journal}{\emph{Advances in Neural Information Processing
  Systems}}  \bibinfo{volume}{34} (\bibinfo{year}{2021}),
  \bibinfo{pages}{14239--14251}.
\newblock


\bibitem[Huang et~al\mbox{.}(2022)]%
        {huang2022learning}
\bibfield{author}{\bibinfo{person}{Jincheng Huang}, \bibinfo{person}{Ping Li},
  \bibinfo{person}{Rui Huang}, {and} \bibinfo{person}{Chen Na}.}
  \bibinfo{year}{2022}\natexlab{}.
\newblock \showarticletitle{Learning heterophilious edge to drop: A general
  framework for boosting graph neural networks}.
\newblock \bibinfo{journal}{\emph{arXiv preprint arXiv:2205.11322}}
  (\bibinfo{year}{2022}).
\newblock


\bibitem[Jin et~al\mbox{.}(2020)]%
        {prognn}
\bibfield{author}{\bibinfo{person}{Wei Jin}, \bibinfo{person}{Yao Ma},
  \bibinfo{person}{Xiaorui Liu}, \bibinfo{person}{Xianfeng Tang},
  \bibinfo{person}{Suhang Wang}, {and} \bibinfo{person}{Jiliang Tang}.}
  \bibinfo{year}{2020}\natexlab{}.
\newblock \showarticletitle{Graph structure learning for robust graph neural
  networks}. In \bibinfo{booktitle}{\emph{Proceedings of the 26th ACM SIGKDD
  international conference on knowledge discovery \& data mining}}.
  \bibinfo{pages}{66--74}.
\newblock


\bibitem[Kearnes et~al\mbox{.}(2016)]%
        {molecular}
\bibfield{author}{\bibinfo{person}{Steven Kearnes}, \bibinfo{person}{Kevin
  McCloskey}, \bibinfo{person}{Marc Berndl}, \bibinfo{person}{Vijay Pande},
  {and} \bibinfo{person}{Patrick Riley}.} \bibinfo{year}{2016}\natexlab{}.
\newblock \showarticletitle{Molecular graph convolutions: moving beyond
  fingerprints}.
\newblock \bibinfo{journal}{\emph{Journal of computer-aided molecular design}}
  \bibinfo{volume}{30}, \bibinfo{number}{8} (\bibinfo{year}{2016}),
  \bibinfo{pages}{595--608}.
\newblock


\bibitem[Kipf and Welling(2017)]%
        {kipf2017gcn}
\bibfield{author}{\bibinfo{person}{Thomas~N. Kipf} {and} \bibinfo{person}{Max
  Welling}.} \bibinfo{year}{2017}\natexlab{}.
\newblock \showarticletitle{Semi-Supervised Classification with Graph
  Convolutional Networks}. In \bibinfo{booktitle}{\emph{5th International
  Conference on Learning Representations, {ICLR} 2017, Toulon, France, April
  24-26, 2017, Conference Track Proceedings}}.
\newblock


\bibitem[Klicpera et~al\mbox{.}(2018)]%
        {appnp}
\bibfield{author}{\bibinfo{person}{Johannes Klicpera},
  \bibinfo{person}{Aleksandar Bojchevski}, {and} \bibinfo{person}{Stephan
  G{\"u}nnemann}.} \bibinfo{year}{2018}\natexlab{}.
\newblock \showarticletitle{Predict then propagate: Graph neural networks meet
  personalized pagerank}.
\newblock \bibinfo{journal}{\emph{arXiv preprint arXiv:1810.05997}}
  (\bibinfo{year}{2018}).
\newblock


\bibitem[Koren(2003)]%
        {spectral_draw}
\bibfield{author}{\bibinfo{person}{Yehuda Koren}.}
  \bibinfo{year}{2003}\natexlab{}.
\newblock \showarticletitle{On spectral graph drawing}. In
  \bibinfo{booktitle}{\emph{International Computing and Combinatorics
  Conference}}. \bibinfo{pages}{496--508}.
\newblock


\bibitem[Lei et~al\mbox{.}(2022)]%
        {evennet}
\bibfield{author}{\bibinfo{person}{Runlin Lei}, \bibinfo{person}{Zhen Wang},
  \bibinfo{person}{Yaliang Li}, \bibinfo{person}{Bolin Ding}, {and}
  \bibinfo{person}{Zhewei Wei}.} \bibinfo{year}{2022}\natexlab{}.
\newblock \showarticletitle{EvenNet: Ignoring Odd-Hop Neighbors Improves
  Robustness of Graph Neural Networks}.
\newblock \bibinfo{journal}{\emph{arXiv preprint arXiv:2205.13892}}
  (\bibinfo{year}{2022}).
\newblock


\bibitem[Li et~al\mbox{.}(2020)]%
        {deeprobust}
\bibfield{author}{\bibinfo{person}{Yaxin Li}, \bibinfo{person}{Wei Jin},
  \bibinfo{person}{Han Xu}, {and} \bibinfo{person}{Jiliang Tang}.}
  \bibinfo{year}{2020}\natexlab{}.
\newblock \showarticletitle{Deeprobust: A pytorch library for adversarial
  attacks and defenses}.
\newblock \bibinfo{journal}{\emph{arXiv preprint arXiv:2005.06149}}
  (\bibinfo{year}{2020}).
\newblock


\bibitem[Liu et~al\mbox{.}(2019)]%
        {attack_1}
\bibfield{author}{\bibinfo{person}{Xuanqing Liu}, \bibinfo{person}{Si Si},
  \bibinfo{person}{Xiaojin Zhu}, \bibinfo{person}{Yang Li}, {and}
  \bibinfo{person}{Cho-Jui Hsieh}.} \bibinfo{year}{2019}\natexlab{}.
\newblock \showarticletitle{A unified framework for data poisoning attack to
  graph-based semi-supervised learning}.
\newblock \bibinfo{journal}{\emph{arXiv preprint arXiv:1910.14147}}
  (\bibinfo{year}{2019}).
\newblock


\bibitem[Ma et~al\mbox{.}(2021)]%
        {ma2021unified}
\bibfield{author}{\bibinfo{person}{Yao Ma}, \bibinfo{person}{Xiaorui Liu},
  \bibinfo{person}{Tong Zhao}, \bibinfo{person}{Yozen Liu},
  \bibinfo{person}{Jiliang Tang}, {and} \bibinfo{person}{Neil Shah}.}
  \bibinfo{year}{2021}\natexlab{}.
\newblock \showarticletitle{A unified view on graph neural networks as graph
  signal denoising}. In \bibinfo{booktitle}{\emph{Proceedings of the 30th ACM
  International Conference on Information \& Knowledge Management}}.
  \bibinfo{pages}{1202--1211}.
\newblock


\bibitem[Ma et~al\mbox{.}(2019)]%
        {attack_2}
\bibfield{author}{\bibinfo{person}{Yao Ma}, \bibinfo{person}{Suhang Wang},
  \bibinfo{person}{Tyler Derr}, \bibinfo{person}{Lingfei Wu}, {and}
  \bibinfo{person}{Jiliang Tang}.} \bibinfo{year}{2019}\natexlab{}.
\newblock \showarticletitle{Attacking graph convolutional networks via
  rewiring}.
\newblock \bibinfo{journal}{\emph{arXiv preprint arXiv:1906.03750}}
  (\bibinfo{year}{2019}).
\newblock


\bibitem[Min et~al\mbox{.}(2020)]%
        {scatterGCN}
\bibfield{author}{\bibinfo{person}{Yimeng Min}, \bibinfo{person}{Frederik
  Wenkel}, {and} \bibinfo{person}{Guy Wolf}.} \bibinfo{year}{2020}\natexlab{}.
\newblock \showarticletitle{Scattering gcn: Overcoming oversmoothness in graph
  convolutional networks}.
\newblock \bibinfo{journal}{\emph{Advances in Neural Information Processing
  Systems}}  \bibinfo{volume}{33} (\bibinfo{year}{2020}),
  \bibinfo{pages}{14498--14508}.
\newblock


\bibitem[Na et~al\mbox{.}(2022)]%
        {nagraph}
\bibfield{author}{\bibinfo{person}{CHEN Na}, \bibinfo{person}{HUANG Jincheng},
  {and} \bibinfo{person}{LI Ping}.} \bibinfo{year}{2022}\natexlab{}.
\newblock \showarticletitle{Graph Neural Network Defense Combined with
  Contrastive Learning}.
\newblock \bibinfo{journal}{\emph{Journal of Frontiers of Computer Science \&
  Technology}} (\bibinfo{year}{2022}).
\newblock


\bibitem[Nt and Maehara(2019)]%
        {gcn_lowpass}
\bibfield{author}{\bibinfo{person}{Hoang Nt} {and} \bibinfo{person}{Takanori
  Maehara}.} \bibinfo{year}{2019}\natexlab{}.
\newblock \showarticletitle{Revisiting graph neural networks: All we have is
  low-pass filters}.
\newblock \bibinfo{journal}{\emph{arXiv preprint arXiv:1905.09550}}
  (\bibinfo{year}{2019}).
\newblock


\bibitem[Pei et~al\mbox{.}(2020)]%
        {geomGCN}
\bibfield{author}{\bibinfo{person}{Hongbin Pei}, \bibinfo{person}{Bingzhe Wei},
  \bibinfo{person}{Kevin~Chen{-}Chuan Chang}, \bibinfo{person}{Yu Lei}, {and}
  \bibinfo{person}{Bo Yang}.} \bibinfo{year}{2020}\natexlab{}.
\newblock \showarticletitle{Geom-GCN: Geometric Graph Convolutional Networks}.
  In \bibinfo{booktitle}{\emph{8th International Conference on Learning
  Representations, {ICLR} 2020, Addis Ababa, Ethiopia, April 26-30, 2020}}.
  \bibinfo{publisher}{OpenReview.net}.
\newblock
\urldef\tempurl%
\url{https://openreview.net/forum?id=S1e2agrFvS}
\showURL{%
\tempurl}


\bibitem[Rozemberczki et~al\mbox{.}(2021)]%
        {Github}
\bibfield{author}{\bibinfo{person}{Benedek Rozemberczki}, \bibinfo{person}{Carl
  Allen}, {and} \bibinfo{person}{Rik Sarkar}.} \bibinfo{year}{2021}\natexlab{}.
\newblock \showarticletitle{Multi-scale attributed node embedding}.
\newblock \bibinfo{journal}{\emph{Journal of Complex Networks}}
  \bibinfo{volume}{9}, \bibinfo{number}{2} (\bibinfo{year}{2021}),
  \bibinfo{pages}{cnab014}.
\newblock


\bibitem[Shuman et~al\mbox{.}(2013)]%
        {graph_fourier}
\bibfield{author}{\bibinfo{person}{David~I Shuman}, \bibinfo{person}{Sunil~K.
  Narang}, \bibinfo{person}{Pascal Frossard}, \bibinfo{person}{Antonio Ortega},
  {and} \bibinfo{person}{Pierre Vandergheynst}.}
  \bibinfo{year}{2013}\natexlab{}.
\newblock \showarticletitle{The emerging field of signal processing on graphs:
  Extending high-dimensional data analysis to networks and other irregular
  domains}.
\newblock \bibinfo{journal}{\emph{IEEE Signal Processing Magazine}}
  \bibinfo{volume}{30}, \bibinfo{number}{3} (\bibinfo{year}{2013}),
  \bibinfo{pages}{83--98}.
\newblock
\urldef\tempurl%
\url{https://doi.org/10.1109/MSP.2012.2235192}
\showDOI{\tempurl}


\bibitem[Tang et~al\mbox{.}(2009)]%
        {filmdata2009}
\bibfield{author}{\bibinfo{person}{Jie Tang}, \bibinfo{person}{Jimeng Sun},
  \bibinfo{person}{Chi Wang}, {and} \bibinfo{person}{Zi Yang}.}
  \bibinfo{year}{2009}\natexlab{}.
\newblock \showarticletitle{Social influence analysis in large-scale networks}.
  In \bibinfo{booktitle}{\emph{Proceedings of the 15th ACM SIGKDD international
  conference on Knowledge discovery and data mining}}.
  \bibinfo{pages}{807--816}.
\newblock


\bibitem[Velickovic et~al\mbox{.}(2018)]%
        {2018gat}
\bibfield{author}{\bibinfo{person}{Petar Velickovic}, \bibinfo{person}{Guillem
  Cucurull}, \bibinfo{person}{Arantxa Casanova}, \bibinfo{person}{Adriana
  Romero}, \bibinfo{person}{Pietro Li{\`{o}}}, {and} \bibinfo{person}{Yoshua
  Bengio}.} \bibinfo{year}{2018}\natexlab{}.
\newblock \showarticletitle{Graph Attention Networks}. In
  \bibinfo{booktitle}{\emph{6th International Conference on Learning
  Representations, {ICLR} 2018, Vancouver, BC, Canada, April 30 - May 3, 2018,
  Conference Track Proceedings}}. \bibinfo{publisher}{OpenReview.net}.
\newblock


\bibitem[Wang and Zhang(2022)]%
        {jacobi-net}
\bibfield{author}{\bibinfo{person}{Xiyuan Wang} {and} \bibinfo{person}{Muhan
  Zhang}.} \bibinfo{year}{2022}\natexlab{}.
\newblock \showarticletitle{How Powerful are Spectral Graph Neural Networks}.
\newblock \bibinfo{journal}{\emph{arXiv preprint arXiv:2205.11172}}
  (\bibinfo{year}{2022}).
\newblock


\bibitem[Wu et~al\mbox{.}(2019a)]%
        {2019sgc}
\bibfield{author}{\bibinfo{person}{Felix Wu}, \bibinfo{person}{Amauri Souza},
  \bibinfo{person}{Tianyi Zhang}, \bibinfo{person}{Christopher Fifty},
  \bibinfo{person}{Tao Yu}, {and} \bibinfo{person}{Kilian Weinberger}.}
  \bibinfo{year}{2019}\natexlab{a}.
\newblock \showarticletitle{Simplifying graph convolutional networks}. In
  \bibinfo{booktitle}{\emph{International conference on machine learning}}.
  PMLR, \bibinfo{pages}{6861--6871}.
\newblock


\bibitem[Wu et~al\mbox{.}(2019b)]%
        {gcn-jaccard}
\bibfield{author}{\bibinfo{person}{Huijun Wu}, \bibinfo{person}{Chen Wang},
  \bibinfo{person}{Yuriy Tyshetskiy}, \bibinfo{person}{Andrew Docherty},
  \bibinfo{person}{Kai Lu}, {and} \bibinfo{person}{Liming Zhu}.}
  \bibinfo{year}{2019}\natexlab{b}.
\newblock \showarticletitle{Adversarial examples on graph data: Deep insights
  into attack and defense}.
\newblock \bibinfo{journal}{\emph{arXiv preprint arXiv:1903.01610}}
  (\bibinfo{year}{2019}).
\newblock


\bibitem[Zhang and Lu(2020)]%
        {la-gcn}
\bibfield{author}{\bibinfo{person}{Li Zhang} {and} \bibinfo{person}{Haiping
  Lu}.} \bibinfo{year}{2020}\natexlab{}.
\newblock \showarticletitle{A feature-importance-aware and robust aggregator
  for GCN}. In \bibinfo{booktitle}{\emph{Proceedings of the 29th ACM
  International Conference on Information \& Knowledge Management}}.
  \bibinfo{pages}{1813--1822}.
\newblock


\bibitem[Zhang et~al\mbox{.}(2022)]%
        {graph_web}
\bibfield{author}{\bibinfo{person}{Wentao Zhang}, \bibinfo{person}{Yu Shen},
  \bibinfo{person}{Zheyu Lin}, \bibinfo{person}{Yang Li},
  \bibinfo{person}{Xiaosen Li}, \bibinfo{person}{Wen Ouyang},
  \bibinfo{person}{Yangyu Tao}, \bibinfo{person}{Zhi Yang}, {and}
  \bibinfo{person}{Bin Cui}.} \bibinfo{year}{2022}\natexlab{}.
\newblock \showarticletitle{Pasca: A graph neural architecture search system
  under the scalable paradigm}. In \bibinfo{booktitle}{\emph{Proceedings of the
  ACM Web Conference 2022}}. \bibinfo{pages}{1817--1828}.
\newblock


\bibitem[Zhang and Zitnik(2020)]%
        {zhang2020gnnguard}
\bibfield{author}{\bibinfo{person}{Xiang Zhang} {and} \bibinfo{person}{Marinka
  Zitnik}.} \bibinfo{year}{2020}\natexlab{}.
\newblock \showarticletitle{Gnnguard: Defending graph neural networks against
  adversarial attacks}.
\newblock \bibinfo{journal}{\emph{Advances in neural information processing
  systems}}  \bibinfo{volume}{33} (\bibinfo{year}{2020}),
  \bibinfo{pages}{9263--9275}.
\newblock


\bibitem[Zhao et~al\mbox{.}(2021)]%
        {heat-gcn}
\bibfield{author}{\bibinfo{person}{Jialin Zhao}, \bibinfo{person}{Yuxiao Dong},
  \bibinfo{person}{Ming Ding}, \bibinfo{person}{Evgeny Kharlamov}, {and}
  \bibinfo{person}{Jie Tang}.} \bibinfo{year}{2021}\natexlab{}.
\newblock \showarticletitle{Adaptive Diffusion in Graph Neural Networks}.
\newblock \bibinfo{journal}{\emph{Advances in Neural Information Processing
  Systems}}  \bibinfo{volume}{34} (\bibinfo{year}{2021}),
  \bibinfo{pages}{23321--23333}.
\newblock


\bibitem[Zhao et~al\mbox{.}(2019)]%
        {traffic}
\bibfield{author}{\bibinfo{person}{Ling Zhao}, \bibinfo{person}{Yujiao Song},
  \bibinfo{person}{Chao Zhang}, \bibinfo{person}{Yu Liu}, \bibinfo{person}{Pu
  Wang}, \bibinfo{person}{Tao Lin}, \bibinfo{person}{Min Deng}, {and}
  \bibinfo{person}{Haifeng Li}.} \bibinfo{year}{2019}\natexlab{}.
\newblock \showarticletitle{T-gcn: A temporal graph convolutional network for
  traffic prediction}.
\newblock \bibinfo{journal}{\emph{IEEE Transactions on Intelligent
  Transportation Systems}} \bibinfo{volume}{21}, \bibinfo{number}{9}
  (\bibinfo{year}{2019}), \bibinfo{pages}{3848--3858}.
\newblock


\bibitem[Zheng et~al\mbox{.}(2020)]%
        {traffic2}
\bibfield{author}{\bibinfo{person}{Chuanpan Zheng}, \bibinfo{person}{Xiaoliang
  Fan}, \bibinfo{person}{Cheng Wang}, {and} \bibinfo{person}{Jianzhong Qi}.}
  \bibinfo{year}{2020}\natexlab{}.
\newblock \showarticletitle{Gman: A graph multi-attention network for traffic
  prediction}. In \bibinfo{booktitle}{\emph{Proceedings of the AAAI conference
  on artificial intelligence}}, Vol.~\bibinfo{volume}{34}.
  \bibinfo{pages}{1234--1241}.
\newblock


\bibitem[Zheng et~al\mbox{.}(2021)]%
        {zheng2021graph}
\bibfield{author}{\bibinfo{person}{Qinkai Zheng}, \bibinfo{person}{Xu Zou},
  \bibinfo{person}{Yuxiao Dong}, \bibinfo{person}{Yukuo Cen},
  \bibinfo{person}{Da Yin}, \bibinfo{person}{Jiarong Xu}, \bibinfo{person}{Yang
  Yang}, {and} \bibinfo{person}{Jie Tang}.} \bibinfo{year}{2021}\natexlab{}.
\newblock \showarticletitle{Graph robustness benchmark: Benchmarking the
  adversarial robustness of graph machine learning}.
\newblock \bibinfo{journal}{\emph{arXiv preprint arXiv:2111.04314}}
  (\bibinfo{year}{2021}).
\newblock


\bibitem[Zhu et~al\mbox{.}(2019)]%
        {rgcn}
\bibfield{author}{\bibinfo{person}{Dingyuan Zhu}, \bibinfo{person}{Ziwei
  Zhang}, \bibinfo{person}{Peng Cui}, {and} \bibinfo{person}{Wenwu Zhu}.}
  \bibinfo{year}{2019}\natexlab{}.
\newblock \showarticletitle{Robust graph convolutional networks against
  adversarial attacks}. In \bibinfo{booktitle}{\emph{Proceedings of the 25th
  ACM SIGKDD international conference on knowledge discovery \& data mining}}.
  \bibinfo{pages}{1399--1407}.
\newblock


\bibitem[Zhu et~al\mbox{.}(2022)]%
        {zhu2022does}
\bibfield{author}{\bibinfo{person}{Jiong Zhu}, \bibinfo{person}{Junchen Jin},
  \bibinfo{person}{Donald Loveland}, \bibinfo{person}{Michael~T Schaub}, {and}
  \bibinfo{person}{Danai Koutra}.} \bibinfo{year}{2022}\natexlab{}.
\newblock \showarticletitle{How does Heterophily Impact the Robustness of Graph
  Neural Networks? Theoretical Connections and Practical Implications}. In
  \bibinfo{booktitle}{\emph{Proceedings of the 28th ACM SIGKDD Conference on
  Knowledge Discovery and Data Mining}}. \bibinfo{pages}{2637--2647}.
\newblock


\bibitem[Z{\"u}gner et~al\mbox{.}(2018)]%
        {nettack}
\bibfield{author}{\bibinfo{person}{Daniel Z{\"u}gner}, \bibinfo{person}{Amir
  Akbarnejad}, {and} \bibinfo{person}{Stephan G{\"u}nnemann}.}
  \bibinfo{year}{2018}\natexlab{}.
\newblock \showarticletitle{Adversarial attacks on neural networks for graph
  data}. In \bibinfo{booktitle}{\emph{Proceedings of the 24th ACM SIGKDD
  international conference on knowledge discovery \& data mining}}.
  \bibinfo{pages}{2847--2856}.
\newblock


\bibitem[Z{\"{u}}gner and G{\"{u}}nnemann(2019)]%
        {metattack}
\bibfield{author}{\bibinfo{person}{Daniel Z{\"{u}}gner} {and}
  \bibinfo{person}{Stephan G{\"{u}}nnemann}.} \bibinfo{year}{2019}\natexlab{}.
\newblock \showarticletitle{Adversarial Attacks on Graph Neural Networks via
  Meta Learning}. In \bibinfo{booktitle}{\emph{7th International Conference on
  Learning Representations, {ICLR} 2019, New Orleans, LA, USA, May 6-9, 2019}}.
\newblock


\bibitem[Z{\"u}gner and G{\"u}nnemann(2019)]%
        {feature-attack2}
\bibfield{author}{\bibinfo{person}{Daniel Z{\"u}gner} {and}
  \bibinfo{person}{Stephan G{\"u}nnemann}.} \bibinfo{year}{2019}\natexlab{}.
\newblock \showarticletitle{Certifiable robustness and robust training for
  graph convolutional networks}. In \bibinfo{booktitle}{\emph{Proceedings of
  the 25th ACM SIGKDD International Conference on Knowledge Discovery \& Data
  Mining}}. \bibinfo{pages}{246--256}.
\newblock


\end{thebibliography}
\clearpage
\appendix
\section{Proof}

\subsection{Proof for Lemma~\ref{lem:dt}}
\begin{proof}
	Suppose we have a node $u$, $v \in \mathcal{N}_{u}$ are the first-order neighbor of $u$. Therefore, we can know that $t \in \mathcal{N}_{u}$.

	The more common neighbors between $u$ and $v$, the larger the expectation value of $\sum_{t\in \mathcal{N}_{u} \text{and} t\in \mathcal{N}_{v}}\frac{1}{d_{t}}$, and the largest case is that the remaining nodes are the common neighbors of u and its first-order neighbors.
	Then the graph is a fully connected graph, if there are $N$ nodes, and its expectation is
	$$\mathbb{E}(\sum_{t\in \mathcal{N}_{u} \text{and} t\in \mathcal{N}_{v}}\frac{1}{d_{t}}) = \frac{\sum^{N-2}\sum^{N-2}\frac{1}{N-1}}{N-2} = \frac{N-2}{N-1} < 1,$$
	Therefore, for any node $u$ and its first-order neighbors in the graph, $\mathbb{E}(\sum_{t\in \mathcal{N}_{u} \text{and} t\in \mathcal{N}_{v}}\frac{1}{d_{t}})<1$, and the proof is completed. 
\end{proof}

\subsection{Proof for Theorem~\ref{th:2}}
\label{proof:th2}
\begin{proof}
	Firstly, let's introduce two commonly used lemmas~\cite{spectral_draw}
	\begin{myLem}\label{lem:spectral_draw}
		\cite{spectral_draw}. Let $U_{y}$ be a generalized eigenvector of $\widehat{A}$, with associated eigenvalue $\lambda$. Then, for each $i$, the exact deviation from the centroid of neighbors is,
	\end{myLem}
	\begin{equation}
		U_{y,i} - \frac{\sum_{j\in \mathcal{N}(i)}\widehat{A}_{i,j}U_{y,j}}{d_{i}} = (1-\lambda)U_{y, i},
	\end{equation}
	
	\begin{myLem}
		\label{lem:sig}
		For the graph signal, consider connected nodes whose signal response strengths are $U_{y,u}$ and $U_{y,v}$, respectively. $\mathbb{E}(U_{l,u}U_{l,v})>0$ holds for low-frequency signals, and $\mathbb{E}(U_{h,u}U_{h,v})<0$ holds for high-frequency ones. For the mid-frequency signals, we have $\left | \mathbb{E}(U_{m,u}U_{m,v})\right | < \left | \mathbb{E}(U_{l,u}U_{l,v})\right |$ and $\left | \mathbb{E}(U_{m,u}U_{m,v})\right | < \left | \mathbb{E}(U_{h,u}U_{h,v})\right |$.
	\end{myLem}
	\noindent which is proved in Appendix~\ref{ap:sig}. \\
	\indent
	The one-edge insertion case is similar to one-edge deletion with an opposite value of $\Delta \lambda_{y}$. We mainly discuss  $|\Delta\lambda|$.Therefore, we can write Lemma~\ref{lem:per} as
	\begin{equation}
		|\Delta \lambda| = |\lambda(U_{y, u}^{2} + U_{y,v}^{2}) - 2U_{y, u}U_{y,v}|
	\end{equation}
	Then, We write $\Delta \lambda$ of high, middle and low frequency, respectively.
	\begin{equation}\label{eq:delta_lambda}
		\begin{aligned}
			|\Delta \lambda_{l}| &= |\lambda(U_{l, u}^{2} + U_{l,v}^{2}) - 2U_{l, u}U_{l,v}|\\
			|\Delta \lambda_{m}| &= |\lambda(U_{m, u}^{2} + U_{m,v}^{2}) - 2U_{m, u}U_{m,v}|\\
			|\Delta \lambda_{h}| &= |\lambda(U_{h, u}^{2} + U_{h,v}^{2}) - 2U_{h, u}U_{h,v}|
		\end{aligned}
	\end{equation}
	
	Then, We transform the equation in lemma~\ref{lem:spectral_draw} into
	$$
	\lambda U_{y,i}^{2}d_{i} = \sum_{j\in \mathcal{N}(i)}\widehat{A}_{i,j}U_{y,j}U_{y,i},
	$$
	
	$$
	\lambda\sum_{i}^{N} U_{y,i}^{2}d_{i} = \sum_{i}^{N}\sum_{j\in \mathcal{N}(i)}\widehat{A}_{i,j}U_{y,j}U_{y,i},
	$$
	
	\begin{equation}\label{eq:el}
		\frac{\lambda\sum_{i}^{N} U_{y,i}^{2}d_{i}}{|E|N} = \frac{\sum_{i}^{N}\sum_{j\in \mathcal{N}(i)}\widehat{A}_{i,j}U_{y,j}U_{y,i}}{|E|N},
	\end{equation}
	
	Next, we can calculate $\mathbb{E}(U_{y, u}^{2} + U_{y,v}^{2})$,
	\begin{equation}\label{eq:u2+u2}
		\begin{aligned}
			\mathbb{E}(U_{y, u}^{2} + U_{y,v}^{2}) &= \frac{\sum^{|E|}_{(u,v) \in E}(U_{y, u}^{2} + U_{y,v}^{2})}{|E|}\\
			&=\frac{\sum_{u}^{N} d_{u}U_{y,u}^{2}}{|E|}\\
		\end{aligned}
	\end{equation}
	where $|E|$ is the number of edge, $d_{u}$ is degree of node $u$ and $D$ is diagonal degree matrix. 
	
	From equations~\ref{eq:el} and~\ref{eq:u2+u2}, we can get
	\begin{equation}\label{eq:key}
		\begin{aligned}
			\frac{\lambda \mathbb{E}(U_{y, u}^{2} + U_{y,v}^{2})}{N} &= \frac{2\mathbb{E}(U_{y, u}U_{y,v})}{2|E|}\\
			\lambda \mathbb{E}(U_{y, u}^{2} + U_{y,v}^{2}) - 2\mathbb{E}(U_{y, u}U_{y,v}) &= 2(\frac{N}{2|E|} - 1)\mathbb{E}(U_{y, u}U_{y,v})\\
			|\lambda \mathbb{E}(U_{y, u}^{2} + U_{y,v}^{2}) - 2\mathbb{E}(U_{y, u}U_{y,v})| &= |2(\frac{N}{2|E|} - 1)\mathbb{E}(U_{y, u}U_{y,v})|
		\end{aligned}
	\end{equation}
	
	We combine equations~\ref{eq:delta_lambda} and~\ref{eq:key} to get
	\begin{equation}\label{eq:final}
		\begin{aligned}
			|\mathbb{E}(\Delta \lambda_{l})| &= |2(\frac{N}{2|E|} - 1)\mathbb{E}(U_{l, u}U_{l,v})|\\
			|\mathbb{E}(\Delta \lambda_{m})| &= |2(\frac{N}{2|E|} - 1)\mathbb{E}(U_{m, u}U_{m,v})|\\
			|\mathbb{E}(\Delta \lambda_{h})| &= |2(\frac{N}{2|E|} - 1)\mathbb{E}(U_{m, u}U_{m,v})|
		\end{aligned}
	\end{equation}
	According to Lemma~~\ref{lem:sig} and formula~\ref{eq:final}, we get $|\mathbb{E}(\Delta \lambda_{m})| < |\mathbb{E}(\Delta \lambda_{l})|$ and $|\mathbb{E}(\Delta \lambda_{m})| < |\mathbb{E}(\Delta \lambda_{h})|$.
	
\end{proof}

\subsection{Proof for Lemma~\ref{lem:sig}}
\label{ap:sig}
\begin{proof}

	Recall that the Rayleigh quotient of a non-zero vector x with respect to a symmetric adjacency matrix $A$ is
	$$ R(A, x) = \frac{x^{T}Ax}{x^{T}x} $$
	We replace $x$ with the eigenvector of $\widehat{A}$ and bring in, then $R(\widehat{A}, U_{y}) = \lambda_{y}$~\cite{spectral1997},
	
	\begin{equation}
		\label{eq:e}
		\lambda_{y} = \frac{\sum_{i} \sum_{j}\widehat{A}_{i,j}U_{y,i}U_{y,j}}{b} =  \frac{\mathbb{E}(U_{y,u}U_{y,v})}{b}\sum_{i}\sum_{j}\widehat{A}_{i,j}
	\end{equation} where $b=U_{y}^{T}U_{y}$ is a constant larger than 0.
	
	For matrix $\widehat{A}$, its eigenvalue $\lambda \in [-1, 1], p \leq \lambda_{l} \leq 1, -1 \leq \lambda_{h}\leq -p,p > \lambda_{m} > -p$, where $p \in (0, r)$. Combining Equation~\ref{eq:e} thus,
	$$\mathbb{E}(U_{l,u}U_{l,v})\geq \frac{bp}{\sum_{i}\sum_{j}\widehat{A}_{i,j}},$$
	$$\mathbb{E}(U_{h,u}U_{h,v})\leq - \frac{bp}{\sum_{i}\sum_{j}\widehat{A}_{i,j}},$$
	$$\frac{bp}{\sum_{i}\sum_{j}A_{i,j}} > \mathbb{E}(U_{m,u}U_{m,v}) > - \frac{bp}{\sum_{i}\sum_{j}\widehat{A}_{i,j}},$$
	
	Thus, we have $\mathbb{E}(U_{l,u}U_{l,v}) > \mathbb{E}(U_{m,u}U_{m,v})> \mathbb{E}(U_{h,u}U_{h,v})$ and $\left | \mathbb{E}(U_{m,u}U_{m,v})\right | < \left | \mathbb{E}(U_{h,u}U_{h,v})\right |$. 
	
	Then, we can get $\left | \mathbb{E}(U_{m,u}U_{m,v})\right | < \left | \mathbb{E}(U_{l,u}U_{l,v})\right |$ and $\left | \mathbb{E}(U_{m,u}U_{m,v})\right | < \left | \mathbb{E}(U_{h,u}U_{h,v})\right |$.
\end{proof}

\subsection{Generalization of Mid-frequency Signals in Homopgily and Heterophily Ring Graph.}
\label{ap:gen}
The adversarial attack on the structure will bring the change of homophily ratio. Let's introduce Spectral regression loss(SRL)~\cite{matengfei}, which associates homophily with spectral domain.
\begin{myDef}
	(Spectral regression loss.) Denote $\alpha=(\alpha_{0},\cdots,\alpha_{N-1})$, $\beta = (\beta_{0},\cdots,\beta_{N-1})$. In a binary node classification task, Spectral Regression Loss (SRL) offilter $g(\Lambda)$ on graph $\mathcal{G}$ is:
\end{myDef}
\begin{equation}
	\begin{aligned}
		L(\mathcal{G})=\sum_{i=0}^{N-1}(\frac{\alpha_{i}}{\sqrt{N}}-\frac{g(\lambda_{i})\beta_{i}}{\sqrt{\sum_{j=0}^{N-1}g(\lambda_{j}^{2})\beta_{j}^{2}}})^{2}\\
		=2-\frac{2}{\sqrt{N}}\sum_{i=0}^{N-1}\frac{\alpha_{i}g(\lambda_{i})\beta_{i}}{\sqrt{\sum_{j=0}^{N-1}g(\lambda_{j}^{2})}\beta_{j}^{2}}
	\end{aligned}
\end{equation}

We know that naive low-pass filters and high-pass filters work better on graphs with certain homophily,but Mid-GCN can more robust under homophily changes. A specific example on ring graphs is given in the following proposition. We see that Mid-GCN intrinsically satisfies the necessary condition for perfect generalization.
\begin{myPro}
	Given two ring graphs, $\mathcal{G}_{1}$ and $\mathcal{G}_{2}$, assuming $h(\mathcal{G}_{1})=0$, $h(\mathcal{G}_{2})=1$(Where $h(\mathcal{G})$ represents the homophily of $\mathcal{G}$), mid-pass filtering can make $L(\mathcal{G}_{1})=L(\mathcal{G}_{2})$, but high and low-pass filters can’t achieve it.
\end{myPro}

\begin{proof}
	
	We can know the necessary condition for a graph filter $g(\lambda)$ to achieve $L(\mathcal{G}_{1}) = L(\mathcal{G}_{2})$ is $g(0) = g(2)$. from~\cite{evennet}. Then, the mid-pass filters obviously easily satisfies the condition, but not with the high/low pass filters.
\end{proof}
\section{Datasets Details}
\label{ap:data-detail}
Details of each graph dataset we utilize during the evaluation are listed here:
\begin{itemize}[leftmargin=*]
	\item \emph{Citation networks}. Including Cora and Citeseer. They are composed of papers as nodes and their
	relationships such as citation relationships, common
	authoring. Node feature is a one-hot vector that
	indicates whether a word is present in that paper.
	Words with frequency less than 10 are removed.
	\item \emph{Github}. A social network
	where nodes correspond to developers and edges refer to mutual
	followers. Node characteristics are location, starred repositories,
	employer, and email address. The task is to classify nodes as web
	developers or machine learning developers. Due to a large amount
	of data, we take one of the subgraphs as the current experimental
	dataset.
	\item \emph{Polblogs}. This dataset only has citation relationships between blogs, and each blog has no features.
	\item \emph{Cora-ML}. Like Cora, this graph is also extracted from the original data of the entire network.
	\item \emph{Film}. Co-occurrence network depicts the co-occurrence relationship of actors. Node features correspond to some keywords in the Wikipedia page. According to the
	actor’s Wikipedia vocabulary, nodes are divided into five
	categories.
\end{itemize}

\section{Hyperparameter Details}
\label{ap:hp}
Table~\ref{tb:hp} reports the detailed hyperparameters of Mid-GCN.
\begin{table}[!ht]
	\centering
	\caption{Hyperparameters of Mid-GCN for reproducibility.
	}
	\label{tb:hp}
	\begin{tabular}{c|cccccc}
		\toprule
		Hyperparameter & Cora & Cite. & Git. & Polb. & Cora-ML & Film\\
		\midrule
		$\alpha$ & 0.5 & 0.55 & 0.55 & 0.2 & 0.2 & 2.0 \\
		Learning rate & 0.01 & 0.01 & 0.01 & 0.01 & 0.01 & 0.01\\
		L2 weight decay& 5e-4 & 5e-4 & 5e-4 & 5e-4 & 5e-4 & 5e-4\\
		Dropout rate & 0.6 & 0.6 & 0.6 & 0.6 & 0.6 & 0.6\\
		Hidden layer size & 128 & 64 & 128 & 64 & 64 & 64\\
		\bottomrule
	\end{tabular}
\end{table}


\end{document}